\documentclass[preprint,12pt,authoryear]{elsarticle}

\usepackage[utf8]{inputenc}
\usepackage[T1]{fontenc}

\usepackage{amsmath}
\usepackage{amssymb}
\usepackage{amsfonts}
\usepackage{amsthm}

\usepackage{booktabs}
\usepackage{multirow}
\usepackage{float}
\usepackage{rotating}
\usepackage{wrapfig}
\usepackage{placeins}

\usepackage{nicefrac}
\usepackage{microtype}
\usepackage{xcolor}
\usepackage{url}
\usepackage{hyperref}
\usepackage{array}
\usepackage{tabularx}

\newtheorem{proposition}{Proposition}
\newtheorem{corollary}{Corollary}
\begin{document}

\begin{frontmatter}



\title{Energy Time-Series Imputation with Differentially Private Diffusion Models via Clipping-Aware Objective Conditioning}
%% Article title

\author[aff1]{Huizhen Huang}
\ead{huanghuizhen@stu.mju.edu.cn}

\author[aff1]{Yu Li}
\ead{liyu@stu.mju.edu.cn}

\author[aff1]{Tao Huang\corref{cor1}}
\ead{huang-tao@mju.edu.cn}

\author[aff1]{Chen Hou}
\ead{houchen@mju.edu.cn}

\cortext[cor1]{Corresponding author}

\affiliation[aff1]{
    organization={School of Computer and Big Data, Minjiang University},
    addressline={No. 200 Xiyuangong Road},
    city={Fuzhou},
    postcode={350108},
    state={Fujian},
    country={China}
}

%% Abstract
\begin{abstract}
Reliable recovery of missing measurements is important for monitoring and
analysis in energy time-series systems, where fine-grained measurements may
contain sensitive temporal information. Diffusion models trained with
differentially private stochastic gradient descent (DP-SGD) provide a
promising framework for privacy-sensitive energy time-series imputation.
Under cosine diffusion schedules, late timesteps correspond to low
signal-to-noise ratio (SNR) conditions, where standard $\varepsilon$-prediction
can induce large pre-clipping gradients. Such gradients are more likely to be
clipped, reducing the retained optimization signal.
The artificial intelligence (AI) contribution lies in formulating this
objective--clipping interaction as an objective optimization problem under
fixed-threshold DP-SGD and developing timestep-aware objective conditioning
for diffusion-based energy time-series imputation. The method adopts
$v$-prediction to mitigate late-timestep gradient amplification, uses static
loss weighting as a uniform-scaling control, and introduces
diffusion-schedule-aware dynamic weighting for stronger attenuation before
clipping. For the engineering application, we evaluate the method on five real-world
energy time-series datasets across random point missingness, contiguous block
missingness, persistent outages, and multiple missing-data severities. Under
matched DP-SGD settings, the proposed method consistently improves
imputation utility over the $\varepsilon$-prediction baseline. Gradient
diagnostics reveal lower upper-tail pre-clipping gradient norms, reduced
clipping fractions, and stronger attenuation at late low-SNR timesteps,
supporting the effectiveness of clipping-aware objective conditioning for
energy time-series imputation.
\end{abstract}

% \centering
% \includegraphics[width=0.95\linewidth]{grabs}

% \begin{highlights}

% \item Objective conditioning reduces gradient clipping in private diffusion training.

% \item Velocity prediction provides the largest utility gain via target reparameterization.

% \item Timestep-aware weighting attenuates gradients more at late diffusion steps.

% \item The method improves imputation across five real-world energy datasets.

% \item Gradient diagnostics show lower upper-tail gradient norms and clipping fractions.

% \end{highlights}

%% Keywords

\end{frontmatter}



\section{Introduction}
\label{sec:introduction}

Time-series data in energy monitoring systems often contain missing
observations because of sensor faults, communication interruptions, device
malfunctions, and signal transmission problems
\citep{lotfipoor2023energy}. In applications such as transformer monitoring,
smart metering, electricity-load monitoring, and building energy management,
imputation is needed to recover missing measurements, while the underlying
temporal records may contain sensitive personal or operational information.
In particular, fine-grained household electricity-consumption profiles can
reveal household activities and occupancy patterns
\citep{molinamarkham2010smartmeter,razavi2019occupancy}. Differential privacy
therefore provides a principled training framework for energy time-series
imputation involving sensitive measurements
\citep{dwork2014algorithmic,mao2024dptsurvey}. Diffusion models are particularly
well suited to this task because they learn conditional denoising processes
for reconstructing missing values
\citep{tashiro2021csdi,alcaraz2023sssd,cao2025timedit}. However, under
differentially private stochastic gradient descent (DP-SGD), the denoising
objective directly interacts with fixed-threshold gradient clipping. When
the objective induces large per-example gradients, clipping attenuates their
contributions to the optimization update and can degrade training utility
\citep{abadi2016deep}.

We identify an objective--clipping interaction that becomes pronounced at
late diffusion timesteps, where diffusion timesteps are distinct from
training iterations. Under cosine schedules, late diffusion timesteps
correspond to low signal-to-noise ratio (SNR) conditions
\citep{nichol2021improved}. In this regime, standard
$\varepsilon$-prediction can induce large pre-clipping gradients that exceed
the fixed clipping threshold, increasing both the frequency and severity of
gradient attenuation. This interaction motivates objective design that
improves the conditioning of diffusion training before DP-SGD clipping.

To address this problem, we develop a timestep-aware objective-conditioning
method for diffusion-based energy time-series imputation trained with
fixed-threshold DP-SGD. The method first adopts $v$-prediction
\citep{salimans2022progressive} as a target reparameterization to mitigate
gradient amplification at late low-SNR diffusion timesteps. We then introduce
static loss weighting as a uniform-scaling control and, building on this
control, develop diffusion-schedule-aware dynamic weighting that applies
stronger objective attenuation toward late low-SNR timesteps before clipping.

The proposed method operates directly at the denoising-objective level while
keeping the clipping-and-noise mechanism, sampling rule, and diffusion
backbone fixed across objective variants. Recent studies have established
the feasibility of training diffusion models with differential privacy across
diverse generative and adaptation settings
\citep{dockhorn2023dpdm,wang2024dppromise,tsai2025dplora}. Building on this
line of research, we investigate how denoising-objective design shapes
pre-clipping gradient behavior and downstream imputation utility in energy
time series.

Experiments on five real-world energy time-series datasets evaluate the
proposed method under matched DP-SGD settings. The engineering evaluation
covers random point missingness, contiguous block missingness, persistent
outages, and multiple missing-data severities. We further examine
suffix-missing forecasting using the same conditional diffusion model,
extending the evaluation to a forecasting setting. A controlled
target--weighting decomposition separates the effects of target
reparameterization and timestep-aware dynamic weighting, while additional
matched DP-SGD settings assess the consistency of these trends across
training conditions. Under the main matched DP-SGD setting, the proposed
method improves imputation utility over the standard
$\varepsilon$-prediction diffusion baseline. Gradient diagnostics further
reveal lower high-percentile pre-clipping gradient norms and reduced overall
clipping fractions, together with stronger gradient attenuation at late
low-SNR diffusion timesteps.

Our contributions are:

\begin{itemize}

    \item \textbf{AI methodological contribution: Objective--clipping problem
    formulation and timestep-aware objective conditioning.}
    We formulate the interaction between cosine-scheduled diffusion objectives
    and fixed-threshold DP-SGD clipping as an objective-design problem in
    energy time-series imputation. We develop a timestep-aware
    objective-conditioning method that combines $v$-prediction as target
    reparameterization, static loss weighting as a uniform-scaling control,
    and diffusion-schedule-aware dynamic weighting to attenuate the objective
    more strongly at late low-SNR diffusion timesteps.

    \item \textbf{AI methodological contribution: Local characterization of
    gradient and clipping behavior.}
    We characterize how deterministic timestep weighting changes pre-clipping
    gradient norms and the effective clipping boundary on the
    unweighted-gradient scale. A controlled target--weighting decomposition
    further separates the effects of target reparameterization and
    timestep-aware dynamic weighting under matched DP-SGD settings.

    \item \textbf{Engineering application contribution: Validation across
    real-world energy time series.}
    We evaluate the proposed method on five real-world energy time-series
    datasets across random point missingness, contiguous block missingness,
    persistent outages, multiple missing-data severities, suffix-missing
    forecasting, and matched DP-SGD settings. The resulting utility
    improvements, together with gradient diagnostics, support the
    effectiveness of timestep-aware objective conditioning across multiple
    energy time-series recovery scenarios.

\end{itemize}

\section{Related Work}
\label{sec:related_work}

\subsection{Private Diffusion Optimization}

Differential privacy in deep learning is commonly implemented through
differentially private stochastic gradient descent (DP-SGD), which applies
fixed-threshold gradient clipping before adding calibrated noise
\citep{abadi2016deep}. Because gradients exceeding the threshold are rescaled
before noise addition, the distribution of pre-clipping gradient norms directly
affects the optimization signal available during private training.

Recent studies have extended diffusion models to differentially private
training. DPDM investigates diffusion-model training under differential
privacy \citep{dockhorn2023dpdm}, DP-Promise improves private diffusion
probabilistic models for image synthesis \citep{wang2024dppromise}, and
DP-LoRA studies parameter-efficient private fine-tuning of diffusion models
\citep{tsai2025dplora}. Collectively, these studies establish the feasibility
of training diffusion models with differential privacy across diverse
generative and adaptation settings, motivating further investigation of their
optimization behavior under DP-SGD.

The objective--clipping interaction at late low-SNR diffusion timesteps has
received limited investigation in private diffusion training. Under cosine
diffusion schedules \citep{nichol2021improved}, late diffusion timesteps
correspond to low signal-to-noise ratio (SNR) conditions, where standard
$\varepsilon$-prediction can produce large gradients before DP-SGD clipping.
This motivates the investigation of objective conditioning for
fixed-threshold DP-SGD clipping, focusing on how denoising objectives shape
pre-clipping gradient behavior and downstream utility.

\subsection{Objective Parameterization and Weighting}

Conditional diffusion models have been successfully applied to time-series
imputation and interpolation, including CSDI, SSSD, and related masked-value
reconstruction methods
\citep{tashiro2021csdi,alcaraz2023sssd,cao2025timedit}. These approaches
establish diffusion modeling as an effective framework for reconstructing
missing temporal observations and primarily focus on imputation quality,
while their optimization behaviors under DP-SGD clipping are not explicitly
characterized.

In standard diffusion training, the denoising objective is commonly defined
using $\varepsilon$-prediction \citep{ho2020ddpm}. Alternative
parameterizations such as $v$-prediction can improve training stability and
sample quality \citep{salimans2022progressive}, while loss-weighting
strategies such as Min-SNR weighting modify the relative contributions of
diffusion timesteps in non-private training
\citep{hang2023minsnr,karras2022edm}. These objective design strategies
provide effective optimization mechanisms for diffusion models, while their
effects on pre-clipping gradient behavior under fixed-threshold DP-SGD
clipping remain to be characterized.

These findings motivate studying diffusion objective optimization under
fixed-threshold DP-SGD clipping. This perspective considers target
reparameterization, static loss weighting, and timestep-aware dynamic
weighting as objective components that shape pre-clipping gradient behavior
and recovery utility under private training.

Beyond diffusion-based approaches, representative time-series architectures
such as PatchTST \citep{nie2023patchtst}, TimesNet
\citep{wu2023timesnet}, and iTransformer
\citep{liu2024itransformer} have shown strong performance in forecasting and
general time-series modeling. Following the comparison setting adopted in
TimeDiT \citep{cao2025timedit}, these models provide non-diffusion references
for evaluating time-series recovery performance.

\subsection{Energy Time-Series Interpolation and Privacy}

Energy monitoring systems provide an important application scenario for
time-series imputation, where sensor faults, communication interruptions,
device malfunctions, and signal-transmission problems can lead to missing
measurements \citep{lotfipoor2023energy}. Fine-grained energy records can also
contain privacy-sensitive information: smart-meter measurements can reveal
household activities
\citep{molinamarkham2010smartmeter}, and residential electricity-consumption
profiles can be used to infer occupancy patterns
\citep{razavi2019occupancy}. These studies establish accurate imputation and
privacy-aware learning as important requirements in practical energy-data
analysis.

Recent diffusion-based approaches have improved energy-related time-series
recovery by modeling complex temporal dependencies and conditional
generation processes. Meanwhile, energy privacy studies have primarily
focused on disclosure risks and privacy-preserving mechanisms. These
directions demonstrate the importance of both recovery utility and privacy
considerations in energy applications.

However, the interaction between diffusion objective optimization and
fixed-threshold DP-SGD clipping in energy time-series recovery has received
limited attention. Energy monitoring therefore provides a concrete
engineering setting for studying how diffusion objective conditioning
influences pre-clipping gradient behavior and recovery utility under private
training.

Overall, prior work establishes three closely related directions:
differentially private diffusion training, diffusion objective design for
time-series recovery, and energy time-series imputation with privacy
considerations. Their intersection---how cosine-scheduled diffusion objectives
shape pre-clipping gradient behavior under fixed-threshold DP-SGD clipping and
how this behavior relates to recovery utility---remains insufficiently
characterized. This motivates the timestep-aware objective-conditioning
method studied in this work.

\section{Method}
\label{sec:method}

\subsection{Preliminaries}
\label{sec:preliminaries}

\paragraph{Time-series interpolation under DP-SGD}
We consider diffusion-based time-series interpolation trained with DP-SGD.
Let $x_0 \in \mathbb{R}^{T \times D}$ denote a time-series segment, where
$T$ is the temporal length and $D$ denotes the feature dimension.
A binary mask $m \in \{0,1\}^{T \times D}$ indicates observed entries.
The model is conditioned on
\[
c = (x_0 \odot m, m),
\]
and reconstructs the masked values $(1-m)\odot x_0$.

The time series is segmented into non-overlapping windows, with each window
treated as an independent training record for DP-SGD. The privacy unit is
defined at the window level, and no cross-window record aggregation is
performed. The resulting task is masked-value interpolation.

\paragraph{Cosine diffusion schedule}
We use a standard diffusion forward process with a cosine schedule:
\begin{equation}
\label{eq:forward_diffusion}
x_t
=
\sqrt{\bar{\alpha}_t}\,x_0
+
\sqrt{1-\bar{\alpha}_t}\,\varepsilon,
\qquad
\varepsilon \sim \mathcal{N}(0,I),
\end{equation}
where $\bar{\alpha}_t$ denotes the cumulative product of the
signal-retention coefficients and determines the retained signal level at
diffusion timestep $t$.

Under the cosine schedule, $\bar{\alpha}_t$ decreases at late diffusion
timesteps, corresponding to low signal-to-noise ratio (SNR) conditions.
This timestep-dependent signal variation provides the basis for designing
objective conditioning aligned with the diffusion process.

\paragraph{DP-SGD clipping}
Training uses DP-SGD.
For a private training record $i$, let
$g_i=\nabla_\theta \mathcal{L}_i$ denote its gradient before DP-SGD
clipping.
DP-SGD clips $g_i$ to a fixed threshold $C$ before adding Gaussian noise:
\begin{equation}
\label{eq:dp_clipping}
\tilde{g}_i
=
g_i
\cdot
\min\left(
1,
\frac{C}{\|g_i\|_2}
\right).
\end{equation}
When $\|g_i\|_2>C$, the per-example gradient is rescaled by
$C/\|g_i\|_2$, attenuating its contribution to the aggregated update.
Therefore, the pre-clipping gradient distribution influences the frequency
and magnitude of clipping. Since the clipping threshold and noise mechanism
are fixed, modifying the denoising objective provides a way to condition the gradients before clipping.

\begin{figure}[t]
    \centering
    \includegraphics[width=0.72\linewidth]{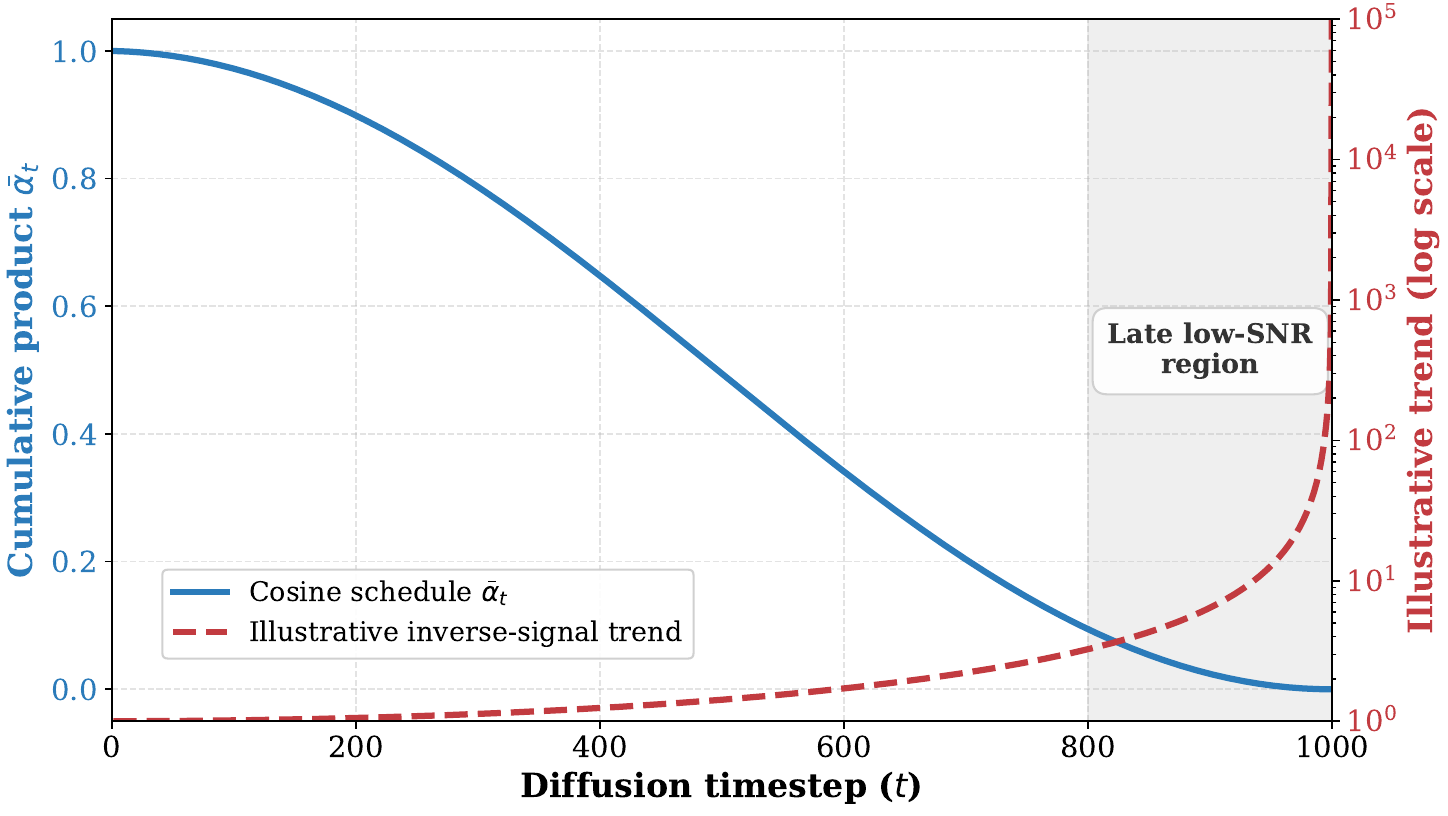}
    \caption{Cosine schedule and late low-SNR region.
    The cumulative signal-retention coefficient $\bar{\alpha}_t$
    decreases at late diffusion timesteps, while the illustrative
    inverse-signal trend increases.}
    \label{fig:tail_sensitivity}
\end{figure}

The objective--clipping interaction becomes pronounced at late diffusion
timesteps.
Figure~\ref{fig:tail_sensitivity} illustrates that these timesteps
correspond to low signal levels under the cosine schedule.
In this region, standard $\varepsilon$-prediction emphasizes noise
prediction and can produce large pre-clipping gradients.
These gradients may exceed the clipping threshold and be attenuated by the
DP-SGD clipping operator, reducing the retained optimization signal.
This motivates conditioning the denoising objective before fixed-threshold
DP-SGD clipping.

\subsection{Our Method: Timestep-Aware Objective Conditioning}
\label{sec:our_method}

The denoising objective provides a mechanism for conditioning gradients
before DP-SGD clipping.
In the core objective comparisons, we keep the backbone, optimizer,
masking protocol, clipping threshold, noise multiplier, and sampling
protocol fixed across objective variants.
The proposed timestep-aware objective-conditioning method modifies only the
prediction target and timestep-dependent weighting.

\paragraph{$v$-prediction as target reparameterization}

The standard objective uses $\varepsilon$-prediction:
\[
\mathcal{L}_{\varepsilon}
=
\mathbb{E}
\left[
\left\|
\varepsilon-\varepsilon_\theta(x_t,t,c)
\right\|_2^2
\right].
\]

We replace it with $v$-prediction:
\[
v
=
\sqrt{\bar{\alpha}_t}\,\varepsilon
-
\sqrt{1-\bar{\alpha}_t}\,x_0,
\]
and optimize
\[
\mathcal{L}_{v}
=
\mathbb{E}
\left[
\left\|
v-v_\theta(x_t,t,c)
\right\|_2^2
\right].
\]

This target reparameterization changes the denoising target representation
and provides a target formulation more suitable for late low-SNR diffusion
timesteps.

\paragraph{Timestep-aware dynamic weighting}

The proposed timestep-aware dynamic weighting uses a timestep-dependent
weight aligned with the diffusion schedule:
\begin{equation}
\label{eq:dynamic_objective}
\mathcal{L}_{\mathrm{dynamic}}
=
\mathbb{E}
\left[
\lambda_t
\left\|
v-v_\theta(x_t,t,c)
\right\|_2^2
\right],
\qquad
\lambda_t=\bar{\alpha}_t.
\end{equation}

Since $\bar{\alpha}_t$ is small at late diffusion timesteps,
timestep-aware dynamic weighting applies stronger attenuation at late
low-SNR diffusion timesteps, where large pre-clipping gradients can occur.

The weight is deterministic once the diffusion schedule is fixed.
It does not depend on private data, gradients, or the DP noise realization.

We choose $\lambda_t=\bar{\alpha}_t$ because this coefficient already
determines the retained signal level in the forward diffusion process,
providing a deterministic timestep-dependent scaling rule aligned with the
diffusion parameterization.

\paragraph{Controlled target--weighting decomposition}

To disentangle target reparameterization and timestep-aware weighting at the
objective level, we define the following factorized objective:
\begin{equation}
\label{eq:factorized_objective}
\mathcal{L}(y,w_t)
=
\mathbb{E}
\left[
w_t
\left\|
y-f_\theta^{(y)}(x_t,t,c)
\right\|_2^2
\right],
\qquad
y\in\{\varepsilon,v\},
\quad
w_t\in\{1,\bar{\alpha}_t\},
\end{equation}
where $f_\theta^{(\varepsilon)}=\varepsilon_\theta$ and
$f_\theta^{(v)}=v_\theta$.

This formulation defines four objective configurations:
$\varepsilon$-prediction without weighting,
$\varepsilon$-prediction with timestep-aware dynamic weighting,
$v$-prediction without weighting, and
$v$-prediction with timestep-aware dynamic weighting.
The decomposition separates target reparameterization from
timestep-dependent weighting at the objective level.
Static loss weighting is considered separately as a uniform-scaling control.

\subsection{Theoretical Analysis}
\label{sec:theory}

This section characterizes the local effect of deterministic timestep
weighting on pre-clipping gradients under a fixed DP-SGD mechanism.
With the clipping threshold $C$, noise multiplier, and sampling rule held
fixed, only the deterministic denoising objective is modified before
clipping.
Additional proof details are provided in
~\ref{app:proof}.

For a private training record $i$ at diffusion timestep $t$, let
$\ell_i(t)$ denote its unweighted denoising loss before timestep weighting.
We consider the weighted objective
\[
\mathcal{L}_{w,i}(t)
=
\lambda_t \ell_i(t),
\qquad
\lambda_t\in(0,1],
\]
where $\lambda_t$ is fixed by the diffusion schedule and does not depend
on private data, gradients, or model parameters.

\begin{proposition}[Deterministic weighting scales pre-clipping gradient magnitude]
\label{prop:gradient_scaling}
For the weighted objective above,
\[
\nabla_\theta \mathcal{L}_{w,i}(t)
=
\lambda_t \nabla_\theta \ell_i(t),
\qquad
\left\|
\nabla_\theta \mathcal{L}_{w,i}(t)
\right\|_2
=
\lambda_t
\left\|
\nabla_\theta \ell_i(t)
\right\|_2.
\]
\end{proposition}

\begin{proof}
Since $\lambda_t$ is deterministic and independent of $\theta$,
\[
\nabla_\theta \mathcal{L}_{w,i}(t)
=
\nabla_\theta
\left(
\lambda_t\ell_i(t)
\right)
=
\lambda_t\nabla_\theta\ell_i(t).
\]
Taking the Euclidean norm and using $\lambda_t>0$ gives the stated
norm identity.
\end{proof}

Proposition~\ref{prop:gradient_scaling} shows that deterministic timestep
weighting rescales the gradient norm before DP-SGD clipping.
Static loss weighting applies uniform scaling, while timestep-aware dynamic
weighting sets $\lambda_t=\bar{\alpha}_t$, producing stronger attenuation at
late low-SNR diffusion timesteps where $\bar{\alpha}_t$ is smaller.

Let
\[
g_i(t)
=
\nabla_\theta\ell_i(t),
\qquad
g_i^{(w)}(t)
=
\lambda_t g_i(t)
\]
denote the unweighted and weighted pre-clipping gradients.

\begin{proposition}[Objective weighting raises the effective clipping boundary on the unweighted-gradient scale]
\label{prop:effective_boundary}
For a fixed diffusion timestep $t$ and clipping threshold $C$, the
weighted example is clipped if and only if
\[
\|g_i(t)\|_2
>
\frac{C}{\lambda_t}.
\]
Consequently, for a fixed distribution of unweighted gradient norms
at timestep $t$,
\[
\mathbb{P}
\left(
\|g_i^{(w)}(t)\|_2>C
\right)
=
\mathbb{P}
\left(
\|g_i(t)\|_2>\frac{C}{\lambda_t}
\right)
\leq
\mathbb{P}
\left(
\|g_i(t)\|_2>C
\right).
\]
\end{proposition}

\begin{proof}
The weighted example is clipped when
\[
\|g_i^{(w)}(t)\|_2
=
\lambda_t\|g_i(t)\|_2
>
C,
\]
which is equivalent to
$\|g_i(t)\|_2>C/\lambda_t$.
Since $\lambda_t\in(0,1]$, we have $C/\lambda_t\geq C$, so
\[
\left\{
\|g_i(t)\|_2>C/\lambda_t
\right\}
\subseteq
\left\{
\|g_i(t)\|_2>C
\right\},
\]
which proves the probability inequality.
\end{proof}

Proposition~\ref{prop:effective_boundary} shows that deterministic
timestep weighting raises the effective clipping boundary when expressed
on the unweighted-gradient scale; the actual DP-SGD clipping threshold
$C$ remains fixed.
For timestep-aware dynamic weighting with
$\lambda_t=\bar{\alpha}_t$, this effect is strongest at late diffusion
timesteps, where $\bar{\alpha}_t$ is smallest.

For a pre-clipping gradient $g$, the clipping operator in
Eq.~\eqref{eq:dp_clipping} can be written as
\[
\tilde{g}
=
g
\cdot
\min\left(
1,
\frac{C}{\|g\|_2}
\right),
\]
and we define the clipping distortion as
\[
d_C(g)
=
\|g-\tilde{g}\|_2.
\]

\begin{corollary}[Gradient-tail excess determines clipping distortion]
\label{cor:tail_excess}
For the DP-SGD clipping operator,
\[
d_C(g)
=
\left(
\|g\|_2-C
\right)_+.
\]
Thus, clipping distortion is determined by the amount by which the
gradient norm exceeds the clipping threshold.
\end{corollary}

Corollary~\ref{cor:tail_excess} motivates the use of high-percentile
pre-clipping gradient norms and overall clipping fractions as
clipping-relevant diagnostics.

The analysis is local to a fixed model state, private training record, and
diffusion timestep. It characterizes deterministic objective scaling and its
effect on pre-clipping gradients under a fixed DP-SGD clipping-and-noise
mechanism, without characterizing global training dynamics across
independently trained objective variants.

\section{Experiments}
\label{sec:experiments}

\subsection{Experimental Setup}
\label{sec:experimental_setup}

We evaluate the proposed timestep-aware objective-conditioning method for
differentially private diffusion-based energy time-series interpolation.
The task is masked-value interpolation: given partially observed inputs, the
model reconstructs the masked values. We use five real-world energy
time-series datasets: ETTh1 and ETTm1 \citep{zhou2021informer},
Electricity \citep{trindade2015electricity}, Individual Household Electric
Power Consumption \citep{hebrail2006household}, and Appliances Energy
Prediction \citep{candanedo2017appliances}. For interpolation, we consider
random point missingness, contiguous block missingness, and persistent
outage missingness at missing ratios of $10\%$, $20\%$, $30\%$, and
$50\%$. These patterns represent different structures of missing energy
measurements, ranging from isolated measurement loss to temporary
communication interruption and prolonged sensor or device unavailability.
We additionally evaluate suffix-missing forecasting as an extended task
setting.

Based on the controlled target--weighting decomposition defined in
Section~\ref{sec:our_method}, we compare four objective variants:
$\varepsilon$-prediction without weighting ($\varepsilon$-pred.),
$\varepsilon$-prediction with timestep-aware dynamic weighting
($\varepsilon$-pred., dynamic), unweighted $v$-prediction
($v$-pred., unweighted), and $v$-prediction with timestep-aware dynamic
weighting ($v$-pred., dynamic). Static loss weighting with
$\lambda\in\{0.6,0.9\}$ is additionally used as a uniform-scaling
control in the progressive objective ablation. Unless otherwise stated,
the core objective variants and static loss weighting controls use the same
DP-SGD protocol and share the same backbone, masking protocol, optimizer,
training steps, sampling rule, clipping threshold, noise multiplier, and
checkpoint selection rule. The main experiments are conducted under a
matched DP-SGD setting with
$(\epsilon,\delta)=(10,10^{-5})$. Sensitivity analyses further evaluate
matched settings at $\epsilon\in\{6,8\}$, while comparisons with external
baselines additionally report results at $\epsilon=8$ alongside the main
$\epsilon=10$ setting.

All predictions are inverse-transformed using the same pre-specified
normalization protocol before computing MSE, MAE, and RMSE. Within each
dataset and evaluation condition, all methods are evaluated on the same
data scale and target positions; interpolation metrics are computed only
on masked positions, and forecasting metrics only over the forecast
horizon. Unless otherwise stated, reported results are averaged over three
random seeds ($42$, $43$, and $44$). Seed-level variability and additional
implementation details are reported in the appendix.

The experiments evaluate interpolation utility and mechanism-related
gradient behavior at complementary levels. MSE, MAE, and RMSE quantify
interpolation utility, while the controlled target--weighting decomposition
separates target reparameterization from timestep-aware dynamic weighting.
The controlled gradient diagnostic examines the timestep-dependent gradient
attenuation characterized by Proposition~\ref{prop:gradient_scaling},
while the training-time gradient diagnostics report high-percentile
pre-clipping gradient norms and overall clipping fractions in relation to
Proposition~\ref{prop:effective_boundary} and
Corollary~\ref{cor:tail_excess}.

\subsection{Main Results}
\label{sec:main_results}

\begin{table}[t]
\centering
\caption{
Main interpolation results under $50\%$ missingness.
Here, $\varepsilon$-pred. denotes $\varepsilon$-prediction, and
$v$-pred., dynamic denotes $v$-prediction with timestep-aware dynamic
weighting.
All methods use the same matched DP-SGD configuration indexed by
$(\epsilon,\delta)=(10,10^{-5})$.
Lower values are better.
Results are averaged over three random seeds ($42$, $43$, $44$).
}
\label{tab:main_results}

\small
\setlength{\tabcolsep}{2.0pt}
\renewcommand{\arraystretch}{1.18}

\resizebox{0.98\linewidth}{!}{
\begin{tabular}{ll cc cc cc}
\toprule
\multirow{2}{*}{\textbf{Dataset}} &
\multirow{2}{*}{\textbf{Missingness Pattern}} &
\multicolumn{2}{c}{\textbf{MSE}} &
\multicolumn{2}{c}{\textbf{MAE}} &
\multicolumn{2}{c}{\textbf{RMSE}} \\
\cmidrule(lr){3-4}
\cmidrule(lr){5-6}
\cmidrule(lr){7-8}
& &
$\varepsilon$-pred. & $v$-pred., dynamic &
$\varepsilon$-pred. & $v$-pred., dynamic &
$\varepsilon$-pred. & $v$-pred., dynamic \\
\midrule

\multirow{3}{*}{ETTh1}
& Random Point
& 1181.36 & \textbf{9.12}
& 29.61 & \textbf{2.36}
& 34.37 & \textbf{3.02} \\
& Contiguous Block
& 867.77 & \textbf{27.91}
& 22.71 & \textbf{4.47}
& 29.46 & \textbf{5.28} \\
& Persistent Outage
& 941.54 & \textbf{20.28}
& 24.85 & \textbf{3.62}
& 30.68 & \textbf{4.48} \\
\midrule

\multirow{3}{*}{ETTm1}
& Random Point
& 1173.89 & \textbf{8.83}
& 30.14 & \textbf{1.56}
& 34.26 & \textbf{2.97} \\
& Contiguous Block
& 1051.89 & \textbf{18.84}
& 27.86 & \textbf{3.21}
& 32.43 & \textbf{4.36} \\
& Persistent Outage
& 1070.37 & \textbf{29.45}
& 27.98 & \textbf{4.56}
& 32.72 & \textbf{5.41} \\
\midrule

\multirow{3}{*}{Electricity}
& Random Point
& 2415.27 & \textbf{32.88}
& 31.90 & \textbf{3.37}
& 49.15 & \textbf{5.73} \\
& Contiguous Block
& 2173.88 & \textbf{31.83}
& 28.74 & \textbf{3.10}
& 46.62 & \textbf{5.64} \\
& Persistent Outage
& 2113.45 & \textbf{25.67}
& 27.95 & \textbf{2.79}
& 45.97 & \textbf{5.07} \\
\midrule

\multirow{3}{*}{Household}
& Random Point
& 61.79 & \textbf{0.25}
& 5.98 & \textbf{0.38}
& 7.86 & \textbf{0.50} \\
& Contiguous Block
& 67.11 & \textbf{0.38}
& 6.40 & \textbf{0.46}
& 8.19 & \textbf{0.62} \\
& Persistent Outage
& 63.03 & \textbf{0.38}
& 6.16 & \textbf{0.46}
& 7.94 & \textbf{0.62} \\
\midrule

\multirow{3}{*}{Appliances}
& Random Point
& 417812.50 & \textbf{53959.11}
& 525.00 & \textbf{151.88}
& 646.38 & \textbf{232.29} \\
& Contiguous Block
& 485956.25 & \textbf{38401.84}
& 575.21 & \textbf{127.80}
& 697.11 & \textbf{195.96} \\
& Persistent Outage
& 505110.42 & \textbf{17996.79}
& 551.88 & \textbf{81.65}
& 710.71 & \textbf{134.15} \\
\bottomrule
\end{tabular}
}
\end{table}

Table~\ref{tab:main_results} compares $\varepsilon$-pred. with
$v$-pred., dynamic under $50\%$ missingness. Across all five datasets
and three missingness patterns, $v$-pred., dynamic achieves lower MSE,
MAE, and RMSE in all 15 dataset--pattern combinations. The improvement
is observed under random point missingness, contiguous block missingness,
and persistent outage missingness, indicating that the utility gain is
observed across missingness patterns with different temporal structures.

Because the two objective variants are evaluated under the same matched
DP-SGD training and evaluation conditions, Table~\ref{tab:main_results}
demonstrates the interpolation utility improvement of the proposed
objective configuration over the standard $\varepsilon$-prediction
baseline within the shared diffusion framework.
Section~\ref{sec:objective_ablation} next separates the contributions of
target reparameterization and timestep-aware dynamic weighting and further
analyzes their relationship with the pre-clipping gradient behavior
characterized in Section~\ref{sec:theory}.

\subsection{Objective Ablation and Gradient Diagnostics}
\label{sec:objective_ablation}

\subsubsection{Progressive Objective Ablation}
\label{sec:progressive_ablation}

\begin{table}[t]
\centering
\caption{
Progressive objective ablation under $50\%$ missingness.
All variants use the same matched DP-SGD configuration indexed by
$(\epsilon,\delta)=(10,10^{-5})$.
Results are averaged over three random seeds ($42$, $43$, $44$).
Lower MSE and MAE are better, and the best result for each dataset,
missingness pattern, and metric is shown in bold.
}
\label{tab:progressive_ablation}
\scriptsize
\setlength{\tabcolsep}{2.0pt}
\renewcommand{\arraystretch}{1.12}
\begin{tabular}{@{}llrrrrrr@{}}
\toprule
\multirow{2}{*}{\textbf{Dataset}} &
\multirow{2}{*}{\textbf{Method}} &
\multicolumn{2}{c}{\textbf{Random Point}} &
\multicolumn{2}{c}{\textbf{Contiguous Block}} &
\multicolumn{2}{c}{\textbf{Persistent Outage}} \\
\cmidrule(lr){3-4}
\cmidrule(lr){5-6}
\cmidrule(lr){7-8}
& &
\textbf{MSE} & \textbf{MAE} &
\textbf{MSE} & \textbf{MAE} &
\textbf{MSE} & \textbf{MAE} \\
\midrule

\multirow{5}{*}{ETTh1}
& $\varepsilon$-pred.
& 1181.36 & 29.61
& 867.77 & 22.71
& 941.54 & 24.85 \\
& $v$-pred., unweighted
& 21.90 & 3.67
& 29.80 & 4.57
& 21.81 & 3.75 \\
& $v$-pred., static ($\lambda=0.6$)
& 19.86 & 3.69
& 49.30 & 6.13
& 47.03 & 5.83 \\
& $v$-pred., static ($\lambda=0.9$)
& 25.07 & 3.96
& 35.68 & 5.03
& 44.37 & 5.72 \\
& $v$-pred., dynamic
& \textbf{9.12} & \textbf{2.36}
& \textbf{27.91} & \textbf{4.47}
& \textbf{20.28} & \textbf{3.62} \\
\midrule

\multirow{5}{*}{ETTm1}
& $\varepsilon$-pred.
& 1173.89 & 30.14
& 1051.89 & 27.86
& 1070.37 & 27.98 \\
& $v$-pred., unweighted
& 10.72 & 2.58
& 29.76 & 4.65
& 54.79 & 6.55 \\
& $v$-pred., static ($\lambda=0.6$)
& 11.93 & 2.78
& 24.35 & 3.61
& 42.21 & 5.80 \\
& $v$-pred., static ($\lambda=0.9$)
& 15.73 & 3.12
& 21.85 & 3.68
& 35.09 & 4.97 \\
& $v$-pred., dynamic
& \textbf{8.83} & \textbf{1.56}
& \textbf{18.84} & \textbf{3.21}
& \textbf{29.45} & \textbf{4.56} \\
\midrule

\multirow{5}{*}{Appliances}
& $\varepsilon$-pred.
& 417812.50 & 525.00
& 485956.25 & 575.21
& 505110.42 & 551.88 \\
& $v$-pred., unweighted
& 66492.21 & 178.55
& 42528.65 & 141.34
& 26032.05 & 107.97 \\
& $v$-pred., static ($\lambda=0.6$)
& 59012.34 & 166.06
& 42397.83 & 137.86
& 21131.38 & 85.58 \\
& $v$-pred., static ($\lambda=0.9$)
& 57985.70 & 171.57
& 39510.84 & 138.79
& 25429.74 & 118.34 \\
& $v$-pred., dynamic
& \textbf{53959.11} & \textbf{151.88}
& \textbf{38401.84} & \textbf{127.80}
& \textbf{17996.79} & \textbf{81.65} \\
\bottomrule
\end{tabular}
\end{table}

Table~\ref{tab:progressive_ablation} examines the progressive
objective-conditioning design under $50\%$ missingness.
Replacing $\varepsilon$-prediction with unweighted $v$-prediction
substantially reduces both MSE and MAE across all reported datasets and
missingness patterns, indicating that target reparameterization provides
the main utility gain within the evaluated objective variants.

Static loss weighting provides the uniform-scaling control introduced
in Section~\ref{sec:our_method}. Its effect varies with the fixed
coefficient and missingness pattern: it improves upon unweighted
$v$-prediction in several settings, but can also increase interpolation
error. Thus, uniform scaling with a single coefficient does not
consistently adapt to the timestep-dependent characteristics of the
diffusion process.

In contrast, $v$-pred., dynamic achieves the lowest MSE and MAE in all
nine reported dataset--pattern conditions. The improvement over
$v$-pred., unweighted is retained under random point missingness,
contiguous block missingness, and persistent outage missingness.
These results support the objective-conditioning progression in
Section~\ref{sec:our_method}: target reparameterization accounts for the
main utility gain, while timestep-aware dynamic weighting provides an
additional utility improvement over unweighted $v$-prediction and the
static loss weighting controls under the main matched DP-SGD configuration.

\subsubsection{Controlled Target--Weighting Decomposition}
\label{sec:target_weighting_decomposition}

To more directly separate target reparameterization from timestep-aware
dynamic weighting, we use the controlled target--weighting decomposition
defined in Section~\ref{sec:our_method}. We compare three controlled
transitions among the four core objective variants: applying
timestep-aware dynamic weighting to $\varepsilon$-prediction, replacing
$\varepsilon$-prediction with unweighted $v$-prediction, and applying
timestep-aware dynamic weighting on top of $v$-prediction.
Table~\ref{tab:target_weighting} summarizes these comparisons over the
12 interpolation conditions formed by three missingness patterns and four
missing ratios. Each entry reports the relative reduction in mean MSE / MAE,
where positive values indicate lower error and negative values indicate
degradation.

\begin{table}[t]
\centering
\caption{
Controlled target--weighting decomposition under the matched DP-SGD
configuration indexed by $(\epsilon,\delta)=(10,10^{-5})$.
Each entry reports the relative reduction (\%) in mean MSE / MAE over
12 interpolation conditions comprising three missingness patterns and four
missing ratios. Positive values indicate improvement, whereas negative
values indicate degradation.
}
\label{tab:target_weighting}

\small
\setlength{\tabcolsep}{4.0pt}
\renewcommand{\arraystretch}{1.15}

\begin{tabular}{lccc}
\toprule
\multirow{4}{*}{\textbf{Dataset}}
&
$\varepsilon$-pred.
&
$\varepsilon$-pred.
&
$v$-pred., unweighted
\\
&
$\downarrow$
&
$\downarrow$
&
$\downarrow$
\\
&
$\varepsilon$-pred., dynamic
&
$v$-pred., unweighted
&
$v$-pred., dynamic
\\
&
\textbf{MSE / MAE}
&
\textbf{MSE / MAE}
&
\textbf{MSE / MAE}
\\
\midrule

ETTh1
& 3.34 / 2.59
& 97.99 / 86.34
& 31.23 / 17.37 \\

ETTm1
& -6.88 / -5.25
& 97.17 / 84.18
& 50.09 / 34.02 \\

Electricity
& 0.57 / 0.56
& 97.76 / 84.77
& 40.63 / 26.09 \\

Household
& 12.04 / 11.26
& 99.13 / 90.95
& 44.04 / 23.78 \\

Appliances
& -19.12 / -15.13
& 91.32 / 75.14
& 8.69 / 12.05 \\

\bottomrule
\end{tabular}
\end{table}

Applying timestep-aware dynamic weighting directly to
$\varepsilon$-prediction produces mixed effects: the mean error decreases
on ETTh1, Electricity, and Household and increases on ETTm1 and Appliances.
This mixed pattern indicates that timestep-aware dynamic weighting alone
does not consistently improve utility when applied to the
$\varepsilon$-prediction objective.

In contrast, replacing $\varepsilon$-prediction with unweighted
$v$-prediction produces substantial reductions in both MSE and MAE on
all five datasets. The relative reduction in mean MSE ranges from
$91.32\%$ to $99.13\%$, while the relative reduction in mean MAE ranges
from $75.14\%$ to $90.95\%$. Compared with the inconsistent effect of
timestep-aware dynamic weighting on $\varepsilon$-prediction, this result
indicates that target reparameterization is the primary source of utility
improvement among the evaluated objective components.

When applied on top of $v$-prediction, timestep-aware dynamic weighting
provides a further improvement on all five datasets under the main
matched DP-SGD configuration. Relative to unweighted $v$-prediction,
the relative reduction in mean MSE ranges from $8.69\%$ to $50.09\%$,
and the relative reduction in mean MAE ranges from $12.05\%$ to
$34.02\%$. Across the progressive objective ablation and controlled
decomposition, the two components show complementary effects:
$v$-prediction provides the main utility gain through target
reparameterization, while timestep-aware dynamic weighting adds further
utility improvement after reparameterization.

\subsubsection{Gradient Diagnostics and Mechanism Analysis}
\label{sec:gradient_diagnostics}

The preceding objective ablations establish the task-level roles of target reparameterization and timestep-aware dynamic weighting. We examine the corresponding pre-clipping gradient behavior in relation to the theoretical mechanism characterized in Section~\ref{sec:theory}. Figure~\ref{fig:gradient_diagnostics} presents an illustrative timestep-wise gradient diagnostic across diffusion timesteps. $\varepsilon$-prediction exhibits pronounced gradient growth
toward the late low-SNR region, while unweighted $v$-prediction shows
smaller gradient magnitudes in this region. The $v$-pred., dynamic curve
shows the strongest attenuation toward late low-SNR diffusion timesteps,
consistent with the stronger attenuation induced by
$\lambda_t=\bar{\alpha}_t$. The figure provides an illustrative
timestep-wise view, while the training-time gradient diagnostics are
reported in Table~\ref{tab:gradient_diagnostics}.

\begin{figure}[t]
\centering
\includegraphics[width=0.98\linewidth]{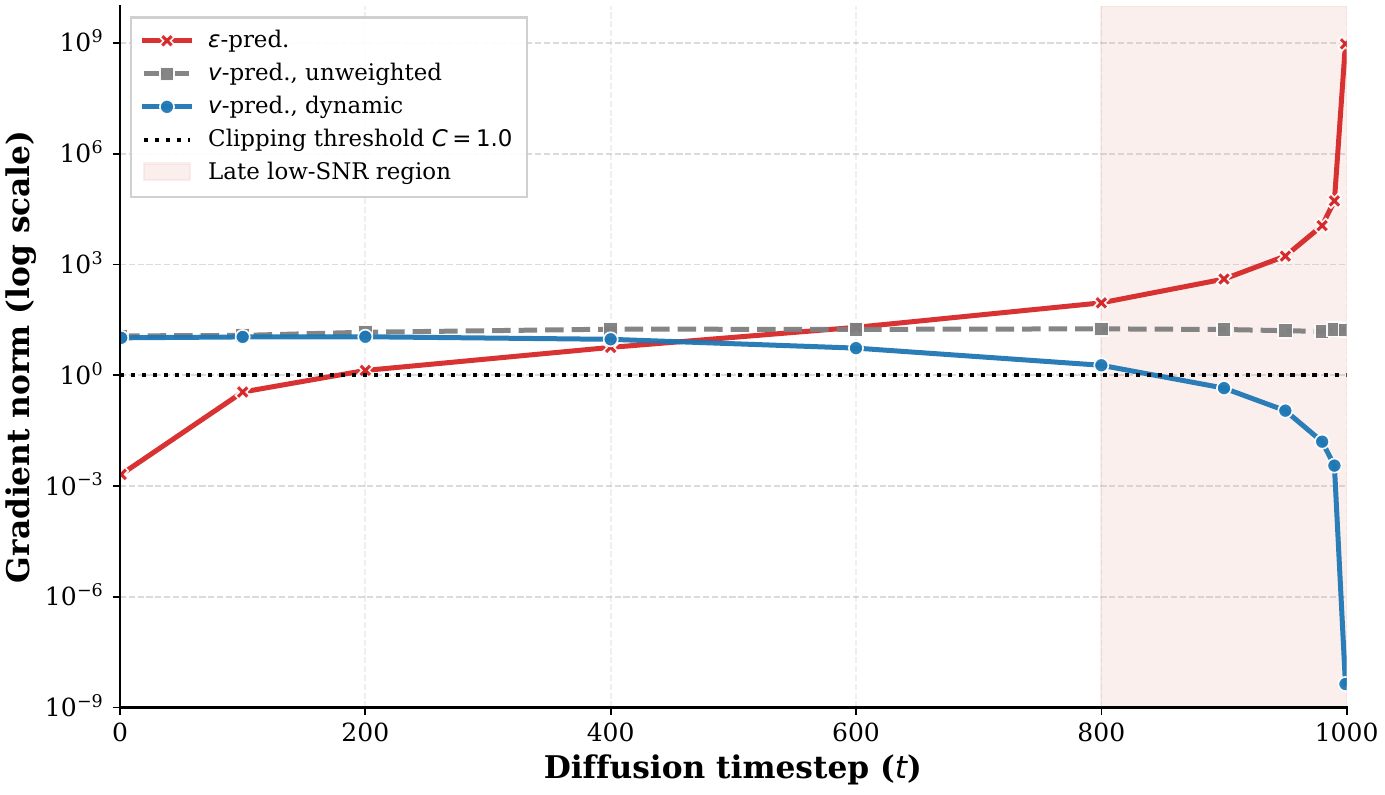}
\caption{
Illustrative timestep-wise gradient diagnostic across diffusion
timesteps. Curves are computed from fixed checkpoints using a synthetic
normalized batch to visualize objective-induced scaling trends.
The dashed line is the clipping threshold $C=1.0$, and the shaded
region marks late low-SNR diffusion timesteps. Training-time gradient
diagnostics are reported in Table~\ref{tab:gradient_diagnostics}.
}
\label{fig:gradient_diagnostics}
\end{figure}

Table~\ref{tab:gradient_diagnostics} reports training-time gradient
diagnostics computed from full-parameter gradient norms for individual
private examples before DP-SGD clipping during the final $20\%$ of
training iterations. Mean and Median summarize the central tendency of
the pre-clipping gradient-norm distribution, while Grad$_{p95}$ and
Grad$_{p99}$ characterize its upper tail. Overall Clipping Fraction
denotes the fraction of recorded gradients exceeding the fixed threshold
$C=1.0$. Relative to $\varepsilon$-prediction, unweighted
$v$-prediction lowers Mean, Median, Grad$_{p95}$, and Grad$_{p99}$ on
all three datasets. The reductions in Grad$_{p95}$ and Grad$_{p99}$
show that target reparameterization is associated with lower
high-percentile pre-clipping gradient norms. The overall clipping
fraction, however, remains 1.000 for unweighted $v$-prediction on all
three datasets.

\begin{table}[t]
\centering
\caption{
Training-time gradient diagnostics under the matched DP-SGD
configuration indexed by $(\epsilon,\delta)=(10,10^{-5})$.
Mean, Median, Grad$_{p95}$, and Grad$_{p99}$ are computed from
full-parameter gradient norms for individual private examples before
DP-SGD clipping during the final $20\%$ of training iterations.
Overall Clipping Fraction denotes the fraction of recorded gradients
exceeding the fixed clipping threshold $C=1.0$.
For the gradient-norm statistics, smaller values indicate lower
pre-clipping gradient magnitudes; for Overall Clipping Fraction,
smaller values indicate less frequent threshold exceedance.
The lowest value within each dataset is shown in bold.
}
\label{tab:gradient_diagnostics}
\small
\setlength{\tabcolsep}{5.0pt}
\renewcommand{\arraystretch}{1.08}

\resizebox{0.98\linewidth}{!}{
\begin{tabular}{llccccc}
\toprule
\textbf{Dataset} &
\textbf{Method} &
\textbf{Mean} &
\textbf{Median} &
\textbf{Grad$_{p95}$} &
\textbf{Grad$_{p99}$} &
\textbf{Overall Clipping Fraction} \\
\midrule

\multirow{3}{*}{ETTh1}
& $\varepsilon$-pred.
& 7.982 & 3.168 & 15.290 & 38.901 & 1.000 \\
& $v$-pred., unweighted
& 4.526 & 3.043 & 9.660 & 13.836 & 1.000 \\
& $v$-pred., dynamic
& \textbf{3.683}
& \textbf{1.943}
& \textbf{8.233}
& \textbf{12.730}
& \textbf{0.944} \\
\midrule

\multirow{3}{*}{ETTm1}
& $\varepsilon$-pred.
& 19.284 & 2.558 & 19.233 & 69.501 & 1.000 \\
& $v$-pred., unweighted
& 5.601 & 2.438 & 12.481 & 21.748 & 1.000 \\
& $v$-pred., dynamic
& \textbf{4.702}
& \textbf{2.136}
& \textbf{10.664}
& \textbf{18.361}
& \textbf{0.948} \\
\midrule

\multirow{3}{*}{Appliances}
& $\varepsilon$-pred.
& 17.040 & 2.989 & 20.847 & 128.333 & 1.000 \\
& $v$-pred., unweighted
& 16.512 & 2.498 & 17.182 & 127.709 & 1.000 \\
& $v$-pred., dynamic
& \textbf{14.575}
& \textbf{1.575}
& \textbf{14.675}
& \textbf{105.080}
& \textbf{0.917} \\
\bottomrule
\end{tabular}
}
\end{table}

Relative to unweighted $v$-prediction, $v$-pred., dynamic further
reduces Mean, Median, Grad$_{p95}$, and Grad$_{p99}$ on all three
datasets. It also reduces the overall clipping fraction from 1.000 to
0.944, 0.948, and 0.917 on ETTh1, ETTm1, and Appliances,
respectively. These training-time gradient diagnostics are consistent
with the clipping-relevant behavior characterized by
Proposition~\ref{prop:effective_boundary} and
Corollary~\ref{cor:tail_excess}: lower pre-clipping gradient norms correspond to a smaller portion of the
gradient distribution extending beyond the fixed clipping threshold. At the same time, the overall clipping fractions
remain above $91\%$, indicating that threshold exceedance remains
frequent despite the reduction.

To directly examine the deterministic timestep weighting described in
Proposition~\ref{prop:gradient_scaling}, we additionally use the
controlled gradient diagnostic. The trained $v$-prediction checkpoint,
inputs, masks, diffusion noise, and timestep realizations are held fixed,
while only the weighting rule is changed from $\lambda_t=1$ to
$\lambda_t=\bar{\alpha}_t$. Across ETTh1, ETTm1, and Appliances, the
median gradient attenuation is only about $0.8\%$--$1.1\%$ in the
$t=0$--$99$ bin, but increases to approximately $98.5\%$--$99.1\%$
in the $t=900$--$999$ bin. The progressively stronger timestep-selective gradient scaling toward late
low-SNR diffusion timesteps provides empirical support for the
timestep-dependent scaling behavior induced by deterministic timestep
weighting, as characterized by Proposition~\ref{prop:gradient_scaling}.
Detailed controlled gradient diagnostics are reported in ~\ref{app:controlled_gradient}.

The controlled gradient diagnostic shows the timestep-selective attenuation induced by deterministic timestep weighting. The training-time gradient diagnostics show lower high-percentile pre-clipping gradient norms and a lower overall clipping fraction for the proposed method. These measurements connect the mechanism characterized in Section~\ref{sec:theory} with the interpolation utility improvements reported in Sections~\ref{sec:main_results} and~\ref{sec:progressive_ablation}, and further characterize the different effects of target reparameterization and timestep-aware dynamic weighting.

\subsection{Robustness to Missing Severity}
\label{sec:missing_severity}

We evaluate the robustness of interpolation utility across different
missing-data severities. Under the same matched DP-SGD configuration
indexed by $(\epsilon,\delta)=(10,10^{-5})$, we vary the missing ratio
at test time from $10\%$ to $50\%$ while holding the trained checkpoints
and the remaining evaluation protocol fixed.

\begin{table}[t]
\centering
\caption{
Robustness to different missing severities under the same matched
DP-SGD configuration indexed by
$(\epsilon,\delta)=(10,10^{-5})$.
MAE is averaged over random point missingness, contiguous block
missingness, and persistent outage missingness.
Results are averaged over three random seeds
($42$, $43$, $44$). Lower values indicate better interpolation
performance, and the best result for each dataset and missing ratio is
shown in bold.
}
\label{tab:missing_severity}
\small
\setlength{\tabcolsep}{6pt}
\renewcommand{\arraystretch}{1.08}

\begin{tabular}{llcccc}
\toprule
\textbf{Dataset} &
\textbf{Method} &
\textbf{10\%} &
\textbf{20\%} &
\textbf{30\%} &
\textbf{50\%} \\
\midrule

\multirow{3}{*}{ETTh1}
& $\varepsilon$-pred.
& 24.68
& 28.50
& 25.55
& 25.72 \\

& $v$-pred., unweighted
& 3.15
& 3.52
& 3.60
& 4.00 \\

& $v$-pred., dynamic
& \textbf{2.35}
& \textbf{2.82}
& \textbf{3.13}
& \textbf{3.48} \\

\midrule

\multirow{3}{*}{ETTm1}
& $\varepsilon$-pred.
& 29.58
& 27.80
& 27.97
& 28.66 \\

& $v$-pred., unweighted
& 4.22
& 4.55
& 4.67
& 4.59 \\

& $v$-pred., dynamic
& \textbf{2.37}
& \textbf{3.11}
& \textbf{3.32}
& \textbf{3.11} \\

\midrule

\multirow{3}{*}{Electricity}
& $\varepsilon$-pred.
& 21.61
& 26.25
& 29.08
& 29.53 \\

& $v$-pred., unweighted
& 4.90
& 4.29
& 3.71
& 3.32 \\

& $v$-pred., dynamic
& \textbf{2.99}
& \textbf{3.07}
& \textbf{2.84}
& \textbf{3.09} \\

\midrule

\multirow{3}{*}{Household}
& $\varepsilon$-pred.
& 6.54
& 6.16
& 5.59
& 6.18 \\

& $v$-pred., unweighted
& 0.43
& 0.59
& 0.59
& 0.61 \\

& $v$-pred., dynamic
& \textbf{0.36}
& \textbf{0.46}
& \textbf{0.42}
& \textbf{0.43} \\

\midrule

\multirow{3}{*}{Appliances}
& $\varepsilon$-pred.
& 453.00
& 489.12
& 559.89
& 550.69 \\

& $v$-pred., unweighted
& 110.00
& 130.29
& 127.35
& 142.62 \\

& $v$-pred., dynamic
& \textbf{104.98}
& \textbf{115.72}
& \textbf{107.61}
& \textbf{120.44} \\

\bottomrule
\end{tabular}
\end{table}

Table~\ref{tab:missing_severity} summarizes interpolation performance
across the four evaluated missing ratios. Across all five datasets,
$v$-pred., dynamic achieves the lowest average MAE in all
20 dataset--severity combinations. Its MAE advantage over both
$\varepsilon$-prediction and unweighted $v$-prediction is maintained
throughout the evaluated range from $10\%$ to $50\%$ missingness.

Unweighted $v$-prediction uses the same prediction target and differs only
in timestep-aware dynamic weighting, providing a direct control for the
weighting effect. The lower average MAE of $v$-pred., dynamic at every
evaluated missing ratio indicates that the benefit of timestep-aware
dynamic weighting is maintained across different missing severities. The
pattern-specific results in Section~\ref{sec:main_results} further show
that the interpolation improvements are observed across different
missingness structures.

\subsection{Extension to Suffix-Missing Forecasting}
\label{sec:suffix_forecasting}

The preceding experiments focus on masked-value interpolation. We extend the evaluation to suffix-missing forecasting within the same conditional diffusion model. For each input window, the observed
prefix is retained as context and the final $H$ missing timesteps are treated as
the prediction target, with forecast horizons
$H\in\{24,48,72\}$. For each objective variant, we reuse its
interpolation-trained checkpoint and change only the evaluation mask for
suffix-missing forecasting. All compared variants use the same matched
DP-SGD configuration indexed by
$(\epsilon,\delta)=(10,10^{-5})$, and forecasting errors are computed
only over the forecast horizon.

\begin{table}[t]
\centering
\caption{
Suffix-missing forecasting performance under different forecast horizons.
All methods use the same matched DP-SGD configuration indexed by
$(\epsilon,\delta)=(10,10^{-5})$.
MSE and MAE are computed only over the forecast horizon.
Results are averaged over three random seeds ($42$, $43$, $44$).
Lower values indicate better forecasting performance, and the best
result for each dataset, forecast horizon, and metric is shown in bold.
}
\label{tab:suffix_forecasting}
\small
\setlength{\tabcolsep}{3.8pt}
\renewcommand{\arraystretch}{1.08}

\resizebox{0.98\linewidth}{!}{
\begin{tabular}{llcccccc}
\toprule
\multirow{2}{*}{\textbf{Dataset}} &
\multirow{2}{*}{\textbf{Method}} &
\multicolumn{3}{c}{\textbf{MSE}} &
\multicolumn{3}{c}{\textbf{MAE}} \\
\cmidrule(lr){3-5}
\cmidrule(lr){6-8}
& &
\textbf{$H=24$} &
\textbf{$H=48$} &
\textbf{$H=72$} &
\textbf{$H=24$} &
\textbf{$H=48$} &
\textbf{$H=72$} \\
\midrule

\multirow{3}{*}{ETTh1}
& $\varepsilon$-pred.
& 785.81 & 981.47 & 885.18
& 22.12 & 25.65 & 23.41 \\
& $v$-pred., unweighted
& 26.02 & 26.28 & 45.55
& 4.17 & 4.20 & 5.97 \\
& $v$-pred., dynamic
& \textbf{17.37} & \textbf{22.80} & \textbf{26.58}
& \textbf{3.77} & \textbf{3.91} & \textbf{4.39} \\
\midrule

\multirow{3}{*}{ETTm1}
& $\varepsilon$-pred.
& 1005.67 & 1033.41 & 1060.82
& 26.55 & 27.25 & 27.91 \\
& $v$-pred., unweighted
& 53.06 & 53.41 & 36.31
& 6.67 & 6.52 & 5.16 \\
& $v$-pred., dynamic
& \textbf{18.98} & \textbf{37.46} & \textbf{33.12}
& \textbf{3.74} & \textbf{5.40} & \textbf{4.99} \\
\midrule

\multirow{3}{*}{Electricity}
& $\varepsilon$-pred.
& 1932.04 & 2294.60 & 2174.27
& 25.60 & 30.33 & 28.75 \\
& $v$-pred., unweighted
& 34.36 & 30.78 & 20.57
& 3.70 & 3.15 & 2.64 \\
& $v$-pred., dynamic
& \textbf{14.94} & \textbf{16.06} & \textbf{20.16}
& \textbf{2.02} & \textbf{2.17} & \textbf{2.53} \\
\midrule

\multirow{3}{*}{Household}
& $\varepsilon$-pred.
& 50.59 & 58.48 & 61.87
& 5.12 & 5.74 & 6.00 \\
& $v$-pred., unweighted
& 0.75 & 0.89 & 0.58
& 0.67 & 0.73 & 0.59 \\
& $v$-pred., dynamic
& \textbf{0.15} & \textbf{0.31} & \textbf{0.30}
& \textbf{0.33} & \textbf{0.41} & \textbf{0.40} \\
\midrule

\multirow{3}{*}{Appliances}
& $\varepsilon$-pred.
& 517675.00 & 560839.58 & 484640.28
& 534.17 & 603.96 & 550.42 \\
& $v$-pred., unweighted
& 10288.46 & 26108.86 & 31083.84
& 82.25 & 109.49 & 121.01 \\
& $v$-pred., dynamic
& \textbf{5526.13} & \textbf{17455.46} & \textbf{25675.11}
& \textbf{56.91} & \textbf{81.86} & \textbf{98.94} \\
\bottomrule
\end{tabular}
}
\end{table}

Table~\ref{tab:suffix_forecasting} shows that both $v$-prediction
variants outperform $\varepsilon$-prediction in MSE and MAE across all
five datasets and all three forecast horizons. Thus, the utility improvement associated with target reparameterization
is maintained when the same conditional diffusion model is evaluated under
suffix-missing forecasting. This result is consistent with the controlled
target--weighting decomposition in
Section~\ref{sec:target_weighting_decomposition}, which characterizes the
contribution of target reparameterization and timestep-aware dynamic
weighting among the evaluated objective components.

Timestep-aware dynamic weighting provides a further improvement over
unweighted $v$-prediction under the main matched DP-SGD configuration.
The $v$-pred., dynamic variant achieves the lowest MSE and MAE in all
15 dataset--horizon combinations at $H=24$, $48$, and $72$. These results indicate that the utility benefit of timestep-aware dynamic
weighting observed in interpolation is maintained in the evaluated
suffix-missing forecasting setting. This extension uses the same
interpolation-trained checkpoints, conditional diffusion backbone, and
objective design, while adapting the evaluation mask to the
suffix-missing forecasting setting.

\subsection{Baseline Comparisons and Fairness Analysis}
\label{sec:baseline_comparisons}

\subsubsection{Comparison with Representative Time-Series Baselines}
\label{sec:representative_baselines}

We further compare the proposed timestep-aware objective-conditioning
method with representative time-series architectures under matched
DP-SGD configurations. Following the random-mask imputation benchmark
adopted in TimeDiT \citep{cao2025timedit}, we evaluate PatchTST, TimesNet,
iTransformer, and the TimeDiT $\varepsilon$-prediction baseline under
$50\%$ random point missingness on ETTh1, ETTm1, and Electricity. Results
are reported for the matched DP-SGD configurations indexed by
$\epsilon=8$ and $\epsilon=10$. Within each configuration, all compared
methods use the same preprocessing, normalization, masking, and evaluation
protocol. The non-diffusion baselines are adapted to the masked-value
interpolation setting using the same observed context and evaluation
protocol.

\begin{table*}[t]
\centering
\caption{
Comparison with representative time-series baselines under matched
DP-SGD interpolation settings at $50\%$ random point missingness.
Results are reported under two matched DP-SGD configurations indexed by
$\epsilon=8$ and $\epsilon=10$. Within each matched DP-SGD configuration,
all compared methods use the same preprocessing, normalization, masking,
and evaluation protocol.
Here, $\epsilon$ indexes the matched DP-SGD configuration,
TimeDiT ($\varepsilon$-pred.) denotes the diffusion baseline with
$\varepsilon$-prediction, and Ours denotes the proposed method with
$v$-prediction and timestep-aware dynamic weighting.
Lower MSE and MAE are better.
Results are averaged over three random seeds ($42$, $43$, $44$).
}
\label{tab:external_baseline}
\small
\setlength{\tabcolsep}{5.5pt}
\renewcommand{\arraystretch}{1.08}

\resizebox{0.98\linewidth}{!}{
\begin{tabular}{ll ccc ccc}
\toprule
\multirow{2}{*}{\textbf{Metric}} &
\multirow{2}{*}{\textbf{Model}} &
\multicolumn{3}{c}{$\epsilon=8$} &
\multicolumn{3}{c}{$\epsilon=10$} \\
\cmidrule(lr){3-5}
\cmidrule(lr){6-8}
& &
\textbf{ETTh1} &
\textbf{ETTm1} &
\textbf{Electricity} &
\textbf{ETTh1} &
\textbf{ETTm1} &
\textbf{Electricity} \\
\midrule

\multirow{5}{*}{\textbf{MSE}}
& PatchTST
& 3349.51 & 1965.84 & 2514.88
& 3348.61 & 1963.74 & 2516.32 \\

& TimesNet
& 2271.85 & 1128.46 & 3213.12
& 2273.45 & 1130.65 & 3214.54 \\

& iTransformer
& 4344.20 & 2966.35 & 3458.10
& 4341.34 & 2968.12 & 3457.21 \\

& TimeDiT ($\varepsilon$-pred.)
& 1216.15 & 958.36 & 2597.01
& 1181.36 & 1173.89 & 2415.27 \\

& \textbf{Ours ($v$-pred., dynamic)}
& \textbf{16.08} & \textbf{36.58} & \textbf{17.99}
& \textbf{9.12} & \textbf{8.83} & \textbf{32.88} \\

\midrule

\multirow{5}{*}{\textbf{MAE}}
& PatchTST
& 49.38 & 37.76 & 33.40
& 49.36 & 37.78 & 33.42 \\

& TimesNet
& 41.22 & 28.84 & 38.35
& 41.24 & 28.86 & 38.37 \\

& iTransformer
& 59.51 & 49.82 & 40.91
& 59.53 & 49.84 & 40.93 \\

& TimeDiT ($\varepsilon$-pred.)
& 30.30 & 25.85 & 34.28
& 29.61 & 30.14 & 31.90 \\

& \textbf{Ours ($v$-pred., dynamic)}
& \textbf{3.07} & \textbf{4.90} & \textbf{2.29}
& \textbf{2.36} & \textbf{1.56} & \textbf{3.37} \\

\bottomrule
\end{tabular}
}
\end{table*}

As shown in Table~\ref{tab:external_baseline}, the proposed method
achieves the lowest MSE and MAE on all three datasets under both matched
DP-SGD configurations, covering all six dataset--configuration
combinations. The advantage is observed over the TimeDiT
$\varepsilon$-prediction baseline as well as PatchTST, TimesNet, and
iTransformer.

This comparison extends the empirical evaluation beyond the
within-backbone objective comparisons in Sections~\ref{sec:main_results}
and~\ref{sec:objective_ablation}. Under the matched DP-SGD
configurations and shared evaluation protocol, the proposed method
achieves consistently lower interpolation error across all evaluated
datasets and configurations. This experiment provides an external performance comparison across
different model architectures. The contribution of target reparameterization and
timestep-aware dynamic weighting is analyzed through controlled
objective comparisons in Section~\ref{sec:objective_ablation}.

\subsubsection{Tuned $\varepsilon$-Prediction Baseline}
\label{sec:tuned_epsilon}

We tune the $\varepsilon$-prediction baseline on ETTh1, ETTm1, and Appliances to assess the effect of baseline tuning on the performance gap to $v$-prediction. The clipping threshold $C$,
learning rate, and checkpoint are selected according to validation loss. The resulting configurations are
shown in Table~\ref{tab:tuned_epsilon}~(a). The selected models are
evaluated under the main matched DP-SGD configuration indexed by
$(\epsilon,\delta)=(10,10^{-5})$ using the same evaluation protocol
as the main experiments.

\begin{table*}[t]
\centering
\caption{
Validation tuning and evaluation of the $\varepsilon$-prediction
baseline.
(a) Validation-selected clipping threshold, learning rate, and
checkpoint for each dataset.
(b) Comparison of the tuned $\varepsilon$-prediction baseline with the
original $\varepsilon$-prediction baseline and unweighted
$v$-prediction using the same evaluation protocol as the main
experiments.
For each missingness pattern, results are averaged over missing ratios
of $10\%$, $20\%$, $30\%$, and $50\%$; for suffix-missing forecasting,
results are averaged over $H\in\{24,48,72\}$.
Lower MAE and RMSE are better, and the best result for each
dataset--condition and metric is shown in bold.
Results in panel (b) are averaged over three random seeds
($42$, $43$, $44$).
}
\label{tab:tuned_epsilon}
\small
\renewcommand{\arraystretch}{1.08}

\textbf{(a) Validation-selected configurations for the tuned
$\varepsilon$-prediction baseline.}

\vspace{3pt}

\setlength{\tabcolsep}{5.0pt}

\resizebox{0.98\linewidth}{!}{
\begin{tabular}{lcccc}
\toprule
\textbf{Dataset} &
\textbf{Selected $C$} &
\textbf{Selected Learning Rate} &
\textbf{Best Validation Loss} &
\textbf{Selected Checkpoint Step} \\
\midrule

ETTh1
& 1
& $3\times10^{-5}$
& 0.203965
& 87000 \\

ETTm1
& 1
& $1\times10^{-4}$
& 0.217985
& 99000 \\

Appliances
& 1
& $3\times10^{-5}$
& 0.180863
& 87000 \\

\bottomrule
\end{tabular}
}

\vspace{8pt}

\textbf{(b) Full comparison with the validation-tuned
$\varepsilon$-prediction baseline.}

\vspace{3pt}

\setlength{\tabcolsep}{4.0pt}

\resizebox{0.98\linewidth}{!}{
\begin{tabular}{ll ccc ccc}
\toprule
\multirow{2}{*}{\textbf{Dataset}} &
\multirow{2}{*}{\textbf{Condition}} &
\multicolumn{3}{c}{\textbf{MAE}} &
\multicolumn{3}{c}{\textbf{RMSE}} \\
\cmidrule(lr){3-5}
\cmidrule(lr){6-8}
& &
\textbf{Tuned $\varepsilon$} &
\textbf{Original $\varepsilon$} &
\textbf{$v$-pred., unweighted} &
\textbf{Tuned $\varepsilon$} &
\textbf{Original $\varepsilon$} &
\textbf{$v$-pred., unweighted} \\
\midrule

\multirow{4}{*}{ETTh1}
& Random Point
& 25.53 & 29.42 & \textbf{3.15}
& 28.74 & 34.16 & \textbf{3.82} \\

& Contiguous Block
& 22.75 & 24.32 & \textbf{4.28}
& 24.23 & 30.79 & \textbf{5.23} \\

& Persistent Outage
& 22.54 & 24.59 & \textbf{3.27}
& 23.73 & 30.04 & \textbf{4.27} \\

& Suffix-Missing Forecasting
& 22.39 & 23.73 & \textbf{4.78}
& 23.71 & 29.70 & \textbf{5.66} \\

\midrule

\multirow{4}{*}{ETTm1}
& Random Point
& 18.86 & 29.65 & \textbf{2.42}
& 24.62 & 33.80 & \textbf{3.03} \\

& Contiguous Block
& 18.63 & 30.86 & \textbf{5.07}
& 24.18 & 34.61 & \textbf{5.98} \\

& Persistent Outage
& 18.47 & 24.99 & \textbf{6.03}
& 24.32 & 30.35 & \textbf{6.89} \\

& Suffix-Missing Forecasting
& 19.08 & 27.23 & \textbf{6.12}
& 24.92 & 32.14 & \textbf{6.87} \\

\midrule

\multirow{4}{*}{Appliances}
& Random Point
& 564.35 & 603.57 & \textbf{159.78}
& 685.49 & 714.83 & \textbf{230.52} \\

& Contiguous Block
& 416.05 & 467.22 & \textbf{154.93}
& 485.03 & 577.96 & \textbf{215.43} \\

& Persistent Outage
& 457.20 & 468.73 & \textbf{67.98}
& 557.42 & 663.37 & \textbf{94.07} \\

& Suffix-Missing Forecasting
& 471.04 & 562.85 & \textbf{104.25}
& 558.76 & 721.52 & \textbf{146.44} \\

\bottomrule
\end{tabular}
}
\end{table*}

Validation-based tuning consistently strengthens the
$\varepsilon$-prediction baseline. As shown in
Table~\ref{tab:tuned_epsilon}~(b), tuned $\varepsilon$-prediction
reduces both MAE and RMSE relative to the original
$\varepsilon$-prediction baseline in all 12 reported
dataset--condition comparisons. Unweighted $v$-prediction retains the lowest MAE and RMSE in every comparison across random point missingness, contiguous block missingness, persistent outage missingness, and suffix-missing forecasting. 

This comparison further examines the contribution of target
reparameterization by comparing unweighted $v$-prediction with a
validation-tuned $\varepsilon$-prediction baseline. The remaining performance gap after validation-based tuning,
together with the controlled target--weighting decomposition in
Section~\ref{sec:target_weighting_decomposition}, shows that the
advantage of target reparameterization over standard
$\varepsilon$-prediction is maintained under stronger baseline tuning.

\section{Conclusion}

We proposed a timestep-aware objective-conditioning method for
differentially private diffusion-based energy time-series interpolation.
The method addresses the objective--clipping interaction that becomes
pronounced at late low-SNR diffusion timesteps, where standard
$\varepsilon$-prediction can produce large gradients before DP-SGD
clipping. Starting from $\varepsilon$-prediction, the objective is
progressively reformulated through $v$-prediction as target
reparameterization, static loss weighting as a uniform-scaling control,
and timestep-aware dynamic weighting that applies stronger objective
attenuation at late low-SNR diffusion timesteps while leaving the
DP-SGD mechanism unchanged.

The controlled decomposition of target reparameterization and weighting
shows that target reparameterization is the primary source of the utility
improvement, while timestep-aware dynamic weighting provides an additional
gain under the main matched DP-SGD configuration. The gradient
diagnostics further support the mechanism characterized in
Section~\ref{sec:theory}. The controlled gradient diagnostic shows
progressively stronger gradient attenuation toward late low-SNR
diffusion timesteps, consistent with the scaling induced by deterministic
timestep weighting, while the training-time gradient diagnostics show
lower high-percentile pre-clipping gradient norms and a lower overall
clipping fraction for the proposed method.

Experiments on five real-world energy time-series datasets show that the
proposed method improves over standard $\varepsilon$-prediction across
random point missingness, contiguous block missingness, and persistent
outage missingness under the main matched DP-SGD configuration. These
improvements are also retained across the evaluated missing-data severities
and in suffix-missing forecasting. Comparisons with representative
time-series baselines and a validation-tuned $\varepsilon$-prediction
baseline further demonstrate the robustness of the observed utility
improvements under matched DP-SGD configurations.
Overall, the results show that clipping-aware objective conditioning can
improve both clipping behavior and interpolation utility under
fixed-threshold DP-SGD, supporting practical energy time-series recovery
under privacy constraints.

\section{Limitations}

The current evaluation uses a cosine diffusion schedule, a fixed diffusion
backbone, and the matched DP-SGD configurations considered in this study.
The incremental gain from timestep-aware dynamic weighting varies with the
training configuration and task setting. Future evaluation should consider
additional diffusion backbones and schedules, alternative private optimization
settings, broader timestep-aware weighting strategies, and dedicated
time-series imputation baselines. The current study adopts a window-level privacy unit, where non-overlapping windows are treated as independent training records.

\appendix

\section{Roadmap}
\label{app:roadmap}

This appendix provides additional details that support the main text
without introducing new claims.

The following appendices report implementation and evaluation details,
including preprocessing, missingness construction, interpolation, and
suffix-missing forecasting protocols; seed variability; external-baseline
and validation-tuned $\varepsilon$-prediction comparison details; full
interpolation results; static loss weighting and timestep-aware dynamic
weighting comparisons; additional proof details; sensitivity to matched
DP-SGD configurations; and the controlled gradient diagnostic.

\section{Implementation Details}
\label{app:impl}

This appendix provides the implementation and training details used in the
core objective comparisons.

Within each dataset, the data are first split chronologically into
$70\%/10\%/20\%$ training, validation, and test subsets. Dataset-specific
missing-value processing is performed independently within each subset
before normalization and window construction. Min--max normalization is
applied per channel using statistics computed from the training subset only;
the resulting parameters are fixed and reused for validation and test.
The normalized series are then partitioned into non-overlapping fixed-length
windows, each serving as an independent training record for DP-SGD. ETTh1,
ETTm1, Household, and Appliances use a window length and stride of 96, whereas
Electricity uses a window length and stride of 192.

For Electricity, we retain the global one-channel construction and use
\texttt{MT\_369} as the fixed evaluation sensor. For Household, the target
variable is \texttt{Global\_active\_power}, retained at the original
one-minute resolution. After the chronological split, interior missing values
are linearly interpolated independently within each subset, while leading and
trailing missing values are removed. Linearly interpolated values are excluded
from the evaluation targets. For Appliances, we use \texttt{Appliances} at
the original ten-minute resolution. No additional resampling is applied to
Household or Appliances.

The diffusion backbone is a DiT-style Transformer conditioned through
AdaLN-Zero. Its architectural configuration is summarized in
Table~\ref{tab:backbone}. The same backbone, optimizer, training masking
protocol, and diffusion training configuration are used across all core
objective variants.
The controlled comparisons differ only in the prediction target and/or loss
weighting.

For each training run, checkpoint selection follows the same
validation-based rule, and the checkpoint with the best validation
performance is used for evaluation. Predictions are inverse-transformed
using the fixed training-derived normalization parameters before computing
MSE, MAE, and RMSE. Within each dataset and evaluation condition, all
compared objective variants use the same data scale, target positions, and
evaluation protocol.

For interpolation, random point missingness, contiguous block missingness,
and persistent outage missingness are applied at evaluation time with
missing ratios of $10\%$, $20\%$, $30\%$, and $50\%$. Artificial missingness
is applied only to originally observed positions, and interpolation metrics
are computed only over the resulting masked targets. Suffix-missing
forecasting uses the same interpolation-trained checkpoints, with the final
$H\in\{24,48,72\}$ positions treated as the prediction target and metrics
computed only over the forecast horizon. No separate forecasting
architecture or forecasting-specific training objective is introduced.

For each evaluation condition, 30 diffusion samples are generated using
DDIM with 50 inference steps. The sample median is used as the point
prediction. Predictions are clamped to $[0,1]$ in normalized space before
inverse transformation.

For stochasticity control, the main experiments are repeated with three
random seeds, 42, 43, and 44, and the reported results are averaged across
the three runs. The same seed set is used across matched objective variants
to control initialization, DP noise, data shuffling, diffusion timestep
sampling, and mask sampling consistently within each comparison.

All experiments were conducted on a high-performance workstation equipped
with NVIDIA A100 80GB PCIe GPUs and 512GB of RAM. The software environment
used Python 3.12.4, PyTorch 2.9.0, and CUDA 12.8.

\begin{table}[H]
\centering
\caption{Hardware and software environment used for all experiments.}
\label{tab:env}
\begin{tabular}{lc}
\toprule
\textbf{Component} & \textbf{Configuration} \\
\midrule
GPU & NVIDIA A100 80GB PCIe \\
System Memory & 512GB RAM \\
Python Version & 3.12.4 \\
Deep Learning Framework & PyTorch 2.9.0 \\
CUDA Version & 12.8 \\
\bottomrule
\end{tabular}
\end{table}

\begin{table}[H]
\centering
\caption{Diffusion-backbone configuration used in the core objective comparisons.}
\label{tab:backbone}

\small
\setlength{\tabcolsep}{5pt}
\renewcommand{\arraystretch}{1.08}

\begin{tabularx}{\linewidth}{
@{}
>{\raggedright\arraybackslash}p{0.30\linewidth}
>{\centering\arraybackslash}X
@{}
}
\toprule
\textbf{Component} & \textbf{Configuration} \\
\midrule

Backbone
& DiT-style Transformer \\

Number of DiT Blocks
& 8 \\

Hidden Dimension
& 256 \\

Self-Attention
& 8 heads \\

Timestep Embedding
& 256-d sinusoidal embedding + MLP \\

Conditioning Mechanism
& AdaLN-Zero \\

Condition Signals
& Timestep/task embeddings, observed-region statistics,
and pooled token representation \\

Trainable Parameters
& $\sim 7.5$M \\

\bottomrule
\end{tabularx}
\end{table}

\begin{table}[H]
\centering
\caption{
Training configuration used for the core objective comparisons.
The main experiments use the matched DP-SGD configuration indexed by
$(\epsilon,\delta)=(10,10^{-5})$.
All core objective variants use the same sampling scheme, clipping
threshold, noise multiplier, and number of optimization steps.
}
\label{tab:hyperparams}

\small
\setlength{\tabcolsep}{5pt}
\renewcommand{\arraystretch}{1.08}

\begin{tabular}{lc}
\toprule
\textbf{Hyperparameter} & \textbf{Value} \\
\midrule

Learning Rate
& $10^{-4}$ \\

Batch Size
& \shortstack{32 (ETTh1, ETTm1, Household, Appliances) \\
16 (Electricity)} \\

Training Steps
& 100,000 \\

Optimizer
& AdamW \\

Diffusion Steps ($T$)
& 1000 \\

Clipping Threshold ($C$)
& 1.0 \\

Noise Multiplier ($\sigma$)
& 1.5 \\

Matched DP-SGD Configuration
& $(\epsilon,\delta)=(10,10^{-5})$ \\

\bottomrule
\end{tabular}
\end{table}

\section{Seed Variability}
\label{app:seed}

The main experimental tables report averages over three random seeds
(42, 43, and 44). To avoid cluttering the main text,
Table~\ref{tab:seed_range} reports the corresponding min--max ranges
across seeds for the progressive objective ablation under $50\%$
missingness.

We use the same datasets, objective variants, and missingness patterns
as in Table~\ref{tab:progressive_ablation}. Each entry reports the MSE
min--max range and the MAE min--max range across the three seeds. The
corresponding seed averages are reported in the main-text ablation table.

\begin{table}[H]
\centering
\caption{
Seed variability for the progressive objective ablation under $50\%$
missingness. The MSE and MAE rows report the corresponding min--max
ranges across three random seeds (42, 43, and 44).
The main-text ablation table reports the corresponding seed averages.
}
\label{tab:seed_range}

\scriptsize
\setlength{\tabcolsep}{3.0pt}
\renewcommand{\arraystretch}{1.10}

\resizebox{0.98\linewidth}{!}{
\begin{tabular}{@{}lllccc@{}}
\toprule
\textbf{Dataset} &
\textbf{Method} &
\textbf{Metric} &
\textbf{Random Point} &
\textbf{Contiguous Block} &
\textbf{Persistent Outage} \\
\midrule

% =========================
% ETTh1
% =========================
\multirow{10}{*}{ETTh1}
& \multirow{2}{*}{$\varepsilon$-pred.}
& MSE
& 1117.24--1235.33
& 826.55--917.09
& 886.24--976.69 \\
&
& MAE
& 28.39--30.38
& 21.61--23.61
& 23.90--25.95 \\

&
\multirow{2}{*}{$v$-pred., unweighted}
& MSE
& 19.80--23.50
& 26.65--31.90
& 20.00--23.51 \\
&
& MAE
& 3.46--3.84
& 4.22--4.94
& 3.50--3.90 \\

&
\multirow{2}{*}{$v$-pred., static ($\lambda=0.6$)}
& MSE
& 18.40--21.46
& 44.50--53.00
& 43.50--50.13 \\
&
& MAE
& 3.50--3.87
& 5.60--6.45
& 5.40--6.23 \\

&
\multirow{2}{*}{$v$-pred., static ($\lambda=0.9$)}
& MSE
& 22.80--26.71
& 32.80--38.44
& 40.20--47.31 \\
&
& MAE
& 3.70--4.26
& 4.65--5.26
& 5.20--6.06 \\

&
\multirow{2}{*}{$v$-pred., dynamic}
& MSE
& 8.72--9.48
& 25.50--29.43
& 18.50--21.84 \\
&
& MAE
& 2.21--2.49
& 4.20--4.79
& 3.40--3.83 \\

\midrule

% =========================
% ETTm1
% =========================
\multirow{10}{*}{ETTm1}
& \multirow{2}{*}{$\varepsilon$-pred.}
& MSE
& 1102.00--1217.67
& 996.00--1123.67
& 1008.00--1120.11 \\
&
& MAE
& 28.30--31.62
& 26.30--28.98
& 26.30--29.14 \\

&
\multirow{2}{*}{$v$-pred., unweighted}
& MSE
& 9.80--11.76
& 27.30--31.58
& 49.50--58.87 \\
&
& MAE
& 2.35--2.74
& 4.20--4.95
& 5.90--6.95 \\

&
\multirow{2}{*}{$v$-pred., static ($\lambda=0.6$)}
& MSE
& 10.80--12.79
& 22.20--26.15
& 38.40--46.23 \\
&
& MAE
& 2.50--3.04
& 3.30--3.93
& 5.30--6.20 \\

&
\multirow{2}{*}{$v$-pred., static ($\lambda=0.9$)}
& MSE
& 14.30--16.79
& 19.90--23.85
& 31.80--37.67 \\
&
& MAE
& 2.80--3.32
& 3.30--3.98
& 4.50--5.41 \\

&
\multirow{2}{*}{$v$-pred., dynamic}
& MSE
& 8.10--9.59
& 17.20--20.02
& 26.80--31.55 \\
&
& MAE
& 1.40--1.70
& 2.90--3.43
& 4.10--4.88 \\

\midrule

% =========================
% Appliances
% =========================
\multirow{10}{*}{Appliances}
& \multirow{2}{*}{$\varepsilon$-pred.}
& MSE
& 392000.00--439437.50
& 453000.00--507868.75
& 475000.00--532331.26 \\
&
& MAE
& 500.00--552.00
& 545.00--599.63
& 525.00--579.88 \\

&
\multirow{2}{*}{$v$-pred., unweighted}
& MSE
& 61200.00--70276.63
& 39000.00--45185.95
& 24000.00--27696.15 \\
&
& MAE
& 166.00--193.55
& 132.00--149.02
& 99.00--114.91 \\

&
\multirow{2}{*}{$v$-pred., static ($\lambda=0.6$)}
& MSE
& 54800.00--63026.02
& 39200.00--44493.49
& 19500.00--22931.38 \\
&
& MAE
& 155.00--178.18
& 128.00--145.58
& 79.00--90.74 \\

&
\multirow{2}{*}{$v$-pred., static ($\lambda=0.9$)}
& MSE
& 53500.00--61757.10
& 36800.00--42510.84
& 23400.00--26989.22 \\
&
& MAE
& 159.00--182.71
& 129.00--145.37
& 110.00--125.02 \\

&
\multirow{2}{*}{$v$-pred., dynamic}
& MSE
& 49800.00--57577.33
& 35500.00--40405.52
& 16500.00--19390.37 \\
&
& MAE
& 142.00--164.64
& 119.00--134.40
& 75.00--86.95 \\

\bottomrule
\end{tabular}
}
\end{table}

\section{Baseline Comparison Details}
\label{app:baseline_details}

The representative baseline comparison is conducted on ETTh1, ETTm1,
and Electricity under $50\%$ random point missingness. Following the
comparison setting adopted in TimeDiT \citep{cao2025timedit}, we consider
PatchTST, TimesNet, iTransformer, and the TimeDiT
$\varepsilon$-prediction baseline.

Results are reported under the matched DP-SGD configurations indexed by
$\epsilon=8$ and $\epsilon=10$. Within each configuration, the compared
methods use the same preprocessing, normalization, masking, and evaluation
protocol. This provides a common training and evaluation protocol for
comparing the representative time-series baselines with the proposed method.

For the validation-tuned $\varepsilon$-prediction baseline, we perform
a local sweep over the clipping threshold and learning rate on ETTh1,
ETTm1, and Appliances under the main matched DP-SGD configuration
indexed by $(\epsilon,\delta)=(10,10^{-5})$. The search space is
\[
C\in\{0.1,0.3,1,3,10\},
\qquad
\mathrm{LR}\in\{3\times10^{-5},10^{-4},3\times10^{-4}\}.
\]
The clipping threshold, learning rate, and checkpoint are selected using
validation loss only, without reference to the test results. The selected
models are then evaluated using the same evaluation protocol as
reported in Section~\ref{sec:tuned_epsilon}.

\section{Full Interpolation Results}
\label{app:full}

Table~\ref{tab:extended_performance} reports the full progressive
objective ablation across all five datasets under $50\%$ missingness.

Across all five datasets and three missingness patterns, replacing
$\varepsilon$-prediction with $v$-prediction substantially improves
interpolation utility. Static loss weighting provides additional
improvement in some settings, while timestep-aware dynamic weighting
provides a further improvement and achieves the lowest MSE and MAE in
all 15 dataset--pattern combinations under the main matched DP-SGD
configuration.

\begin{table}[H]
\centering
\caption{
Full progressive objective ablation under $50\%$ missingness across all
five datasets. All variants use the main matched DP-SGD configuration
indexed by $(\epsilon,\delta)=(10,10^{-5})$.
MSE and MAE are reported for Random Point, Contiguous Block, and
Persistent Outage missingness.
Bold marks the best result among the evaluated progressive objective
variants for each dataset, missingness pattern, and metric.
Results are averaged over three random seeds (42, 43, and 44).
}
\label{tab:extended_performance}

\scriptsize
\setlength{\tabcolsep}{2.3pt}
\renewcommand{\arraystretch}{1.12}

\resizebox{0.98\linewidth}{!}{
\begin{tabular}{ll cc cc cc}
\toprule
\multirow{2}{*}{\textbf{Dataset}} &
\multirow{2}{*}{\textbf{Method}} &
\multicolumn{2}{c}{\textbf{Random Point}} &
\multicolumn{2}{c}{\textbf{Contiguous Block}} &
\multicolumn{2}{c}{\textbf{Persistent Outage}} \\
\cmidrule(lr){3-4}
\cmidrule(lr){5-6}
\cmidrule(lr){7-8}
& &
\textbf{MSE} & \textbf{MAE} &
\textbf{MSE} & \textbf{MAE} &
\textbf{MSE} & \textbf{MAE} \\
\midrule

\multirow{5}{*}{ETTh1}
& $\varepsilon$-pred.
& 1181.36 & 29.61
& 867.77 & 22.71
& 941.54 & 24.85 \\

& $v$-pred., unweighted
& 21.90 & 3.67
& 29.80 & 4.57
& 21.81 & 3.75 \\

& $v$-pred., static ($\lambda=0.6$)
& 19.86 & 3.69
& 49.30 & 6.13
& 47.03 & 5.83 \\

& $v$-pred., static ($\lambda=0.9$)
& 25.07 & 3.96
& 35.68 & 5.03
& 44.37 & 5.72 \\

& $v$-pred., dynamic
& \textbf{9.12} & \textbf{2.36}
& \textbf{27.91} & \textbf{4.47}
& \textbf{20.28} & \textbf{3.62} \\

\midrule

\multirow{5}{*}{ETTm1}
& $\varepsilon$-pred.
& 1173.89 & 30.14
& 1051.89 & 27.86
& 1070.37 & 27.98 \\

& $v$-pred., unweighted
& 10.72 & 2.58
& 29.76 & 4.65
& 54.79 & 6.55 \\

& $v$-pred., static ($\lambda=0.6$)
& 11.93 & 2.78
& 24.35 & 3.61
& 42.21 & 5.80 \\

& $v$-pred., static ($\lambda=0.9$)
& 15.73 & 3.12
& 21.85 & 3.68
& 35.09 & 4.97 \\

& $v$-pred., dynamic
& \textbf{8.83} & \textbf{1.56}
& \textbf{18.84} & \textbf{3.21}
& \textbf{29.45} & \textbf{4.56} \\

\midrule

\multirow{5}{*}{Electricity}
& $\varepsilon$-pred.
& 2415.27 & 31.90
& 2173.88 & 28.74
& 2113.45 & 27.95 \\

& $v$-pred., unweighted
& 33.92 & 3.50
& 36.05 & 3.42
& 30.71 & 3.05 \\

& $v$-pred., static ($\lambda=0.6$)
& 51.03 & 4.14
& 41.07 & 3.36
& 37.51 & 3.15 \\

& $v$-pred., static ($\lambda=0.9$)
& 36.20 & 3.80
& 33.04 & 3.39
& 30.19 & 2.96 \\

& $v$-pred., dynamic
& \textbf{32.88} & \textbf{3.37}
& \textbf{31.83} & \textbf{3.10}
& \textbf{25.67} & \textbf{2.79} \\

\midrule

\multirow{5}{*}{Household}
& $\varepsilon$-pred.
& 61.79 & 5.98
& 67.11 & 6.40
& 63.03 & 6.16 \\

& $v$-pred., unweighted
& 0.28 & 0.42
& 0.70 & 0.63
& 0.94 & 0.80 \\

& $v$-pred., static ($\lambda=0.6$)
& 0.31 & 0.46
& 0.50 & 0.54
& 0.60 & 0.59 \\

& $v$-pred., static ($\lambda=0.9$)
& 0.27 & 0.39
& 0.41 & 0.49
& 0.52 & 0.57 \\

& $v$-pred., dynamic
& \textbf{0.25} & \textbf{0.38}
& \textbf{0.38} & \textbf{0.46}
& \textbf{0.38} & \textbf{0.46} \\

\midrule

\multirow{5}{*}{Appliances}
& $\varepsilon$-pred.
& 417812.50 & 525.00
& 485956.25 & 575.21
& 505110.42 & 551.88 \\

& $v$-pred., unweighted
& 66492.21 & 178.55
& 42528.65 & 141.34
& 26032.05 & 107.97 \\

& $v$-pred., static ($\lambda=0.6$)
& 59012.34 & 166.06
& 42397.83 & 137.86
& 21131.38 & 85.58 \\

& $v$-pred., static ($\lambda=0.9$)
& 57985.70 & 171.57
& 39510.84 & 138.79
& 25429.74 & 118.34 \\

& $v$-pred., dynamic
& \textbf{53959.11} & \textbf{151.88}
& \textbf{38401.84} & \textbf{127.80}
& \textbf{17996.79} & \textbf{81.65} \\

\bottomrule
\end{tabular}
}
\end{table}

\section{Static Loss Weighting versus Timestep-Aware Dynamic Weighting}
\label{app:static}

Figure~\ref{fig:lambda_sweep} shows validation MSE and gradient norm as
functions of the static loss weighting coefficient $\lambda$.

\begin{figure}[H]
\centering
\includegraphics[width=0.7\linewidth]{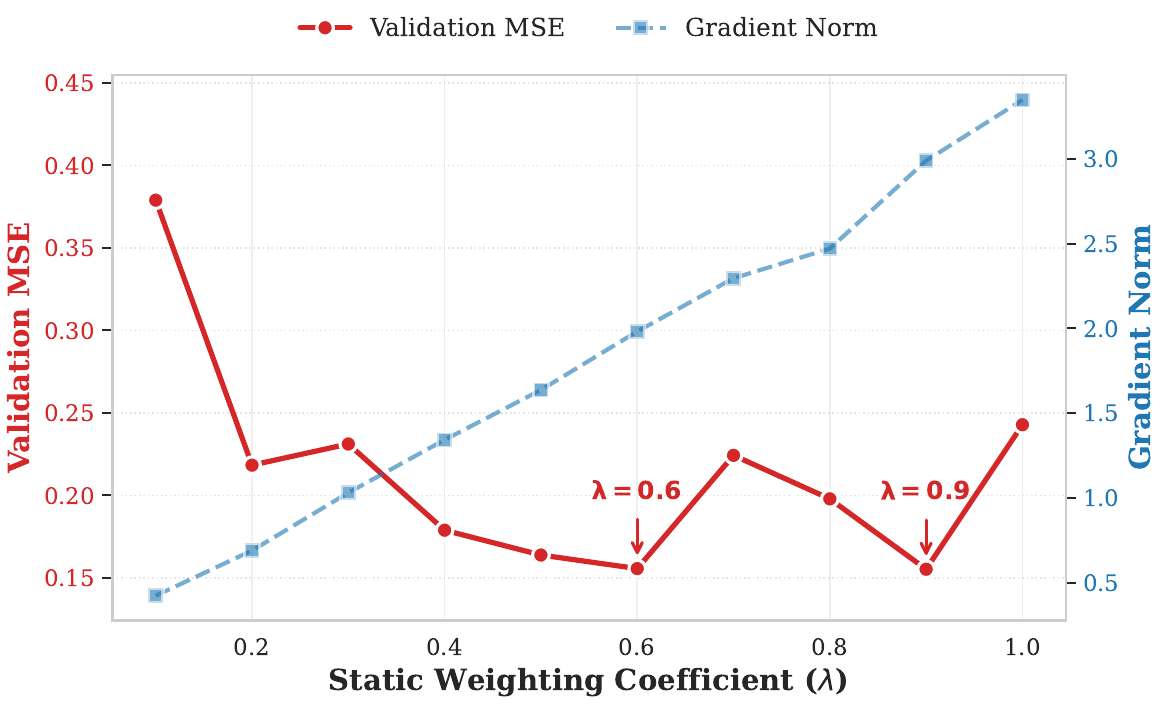}
\caption{
Validation MSE and gradient norm as functions of the static loss
weighting coefficient $\lambda$.
}
\label{fig:lambda_sweep}
\end{figure}

Static loss weighting provides uniform compression across diffusion
timesteps. The static sweep shows that moderate values of $\lambda$ can
improve validation error, but the relationship between validation error
and gradient norm is not strictly monotonic. This motivates timestep-aware dynamic weighting, which follows the
diffusion schedule to provide timestep-dependent compression instead of
a single global coefficient.

\section{Additional Proof Details}
\label{app:proof}

We provide additional details for the local scaling and clipping-aware
analysis in Section~\ref{sec:theory}. The analysis assumes a fixed
DP-SGD clipping threshold $C$ and deterministic timestep weights
$\lambda_t\in(0,1]$ that do not depend on private data, gradients,
model parameters, or the DP noise realization.

For a private example $i$ at diffusion timestep $t$, suppose its
unweighted gradient is locally bounded as
\[
\left\|
\nabla_\theta \ell_i(t)
\right\|_2
\leq G_t.
\]
From Proposition~\ref{prop:gradient_scaling},
\[
\nabla_\theta \mathcal{L}_{w,i}(t)
=
\lambda_t
\nabla_\theta \ell_i(t),
\]
and therefore
\[
\left\|
\nabla_\theta \mathcal{L}_{w,i}(t)
\right\|_2
\leq
\lambda_t G_t.
\]
Static loss weighting sets $\lambda_t=\lambda$, applying the same
compression factor to all diffusion timesteps. Timestep-aware dynamic
weighting sets $\lambda_t=\bar{\alpha}_t$, so the corresponding bound is
reduced more strongly at late low-SNR diffusion timesteps where
$\bar{\alpha}_t$ is small.

We next give the explicit derivation of the clipping-distortion result
used in Corollary~\ref{cor:tail_excess}. For a pre-clipping gradient
$g$, DP-SGD applies
\[
\tilde g
=
g\cdot
\min\left(
1,
\frac{C}{\|g\|_2}
\right).
\]
If $\|g\|_2\leq C$, then $\tilde g=g$ and the clipping distortion is
zero. If $\|g\|_2>C$, then
\[
\tilde g
=
\frac{C}{\|g\|_2}g,
\]
and hence
\[
\begin{aligned}
\|g-\tilde g\|_2
&=
\left\|
g-\frac{C}{\|g\|_2}g
\right\|_2 \\
&=
\left(
1-\frac{C}{\|g\|_2}
\right)
\|g\|_2 \\
&=
\|g\|_2-C.
\end{aligned}
\]
Combining the two cases gives
\[
d_C(g)
=
\|g-\tilde g\|_2
=
(\|g\|_2-C)_+.
\]

These derivations complement the local scaling and clipping-aware
analysis in Section~\ref{sec:theory} by further characterizing gradient
behavior under a fixed DP-SGD clipping-and-noise mechanism.

\section{Sensitivity to Matched DP-SGD Configurations}
\label{app:privacy}

Table~\ref{tab:privacy_sensitivity} reports the sensitivity results
under matched DP-SGD configurations indexed by
$\epsilon\in\{6,8,10\}$. For each dataset and missing ratio, MAE is
averaged over random point missingness, contiguous block missingness,
and persistent outage missingness.

\begin{table}[H]
\centering
\caption{
Interpolation MAE under matched DP-SGD configurations indexed by
$\epsilon\in\{6,8,10\}$.
Here, $\epsilon$ indexes the matched DP-SGD configuration,
$\varepsilon$-pred. denotes $\varepsilon$-prediction, and
$v$-pred., dynamic denotes $v$-prediction with timestep-aware dynamic
weighting $(\lambda_t=\bar{\alpha}_t)$.
For each dataset and missing ratio, MAE is averaged over random point
missingness, contiguous block missingness, and persistent outage
missingness.
Lower values are better, and the lower value within each matched
configuration is shown in bold.
Results are averaged over three random seeds ($42$, $43$, $44$).
}
\label{tab:privacy_sensitivity}

\small
\setlength{\tabcolsep}{4.0pt}
\renewcommand{\arraystretch}{1.08}

\resizebox{0.98\linewidth}{!}{
\begin{tabular}{lc cc cc cc}
\toprule
\multirow{2}{*}{\textbf{Dataset}} &
\multirow{2}{*}{\textbf{Missing Ratio}} &
\multicolumn{2}{c}{\textbf{$\epsilon=6$}} &
\multicolumn{2}{c}{\textbf{$\epsilon=8$}} &
\multicolumn{2}{c}{\textbf{$\epsilon=10$}} \\
\cmidrule(lr){3-4}
\cmidrule(lr){5-6}
\cmidrule(lr){7-8}
& &
$\varepsilon$-pred. & $v$-pred., dynamic &
$\varepsilon$-pred. & $v$-pred., dynamic &
$\varepsilon$-pred. & $v$-pred., dynamic \\
\midrule

\multirow{4}{*}{ETTh1}
& 10\%
& 24.98 & \textbf{5.25}
& 26.28 & \textbf{3.13}
& 24.68 & \textbf{2.35} \\
& 20\%
& 28.57 & \textbf{6.24}
& 27.81 & \textbf{4.38}
& 28.50 & \textbf{2.82} \\
& 30\%
& 26.89 & \textbf{7.66}
& 29.51 & \textbf{4.38}
& 25.55 & \textbf{3.13} \\
& 50\%
& 28.07 & \textbf{8.30}
& 29.68 & \textbf{4.92}
& 25.72 & \textbf{3.48} \\

\midrule

\multirow{4}{*}{ETTm1}
& 10\%
& 26.23 & \textbf{5.04}
& 28.55 & \textbf{5.08}
& 29.58 & \textbf{2.37} \\
& 20\%
& 26.62 & \textbf{4.83}
& 28.38 & \textbf{5.43}
& 27.80 & \textbf{3.11} \\
& 30\%
& 27.26 & \textbf{4.99}
& 27.64 & \textbf{4.85}
& 27.97 & \textbf{3.32} \\
& 50\%
& 26.07 & \textbf{4.97}
& 28.89 & \textbf{4.42}
& 28.66 & \textbf{3.11} \\

\midrule

\multirow{4}{*}{Electricity}
& 10\%
& 34.92 & \textbf{4.65}
& 34.90 & \textbf{2.61}
& 21.61 & \textbf{2.99} \\
& 20\%
& 34.25 & \textbf{4.26}
& 29.58 & \textbf{3.27}
& 26.25 & \textbf{3.07} \\
& 30\%
& 37.80 & \textbf{4.52}
& 31.69 & \textbf{3.21}
& 29.08 & \textbf{2.84} \\
& 50\%
& 37.44 & \textbf{6.14}
& 31.91 & \textbf{2.92}
& 29.53 & \textbf{3.09} \\

\midrule

\multirow{4}{*}{Household}
& 10\%
& 6.87 & \textbf{0.64}
& 5.15 & \textbf{0.34}
& 6.54 & \textbf{0.36} \\
& 20\%
& 7.78 & \textbf{0.60}
& 5.42 & \textbf{0.43}
& 6.16 & \textbf{0.46} \\
& 30\%
& 7.71 & \textbf{0.55}
& 5.46 & \textbf{0.38}
& 5.59 & \textbf{0.42} \\
& 50\%
& 7.46 & \textbf{0.61}
& 5.60 & \textbf{0.42}
& 6.18 & \textbf{0.43} \\

\midrule

\multirow{4}{*}{Appliances}
& 10\%
& 607.33 & \textbf{120.59}
& 387.67 & \textbf{122.62}
& 453.00 & \textbf{104.98} \\
& 20\%
& 674.56 & \textbf{126.50}
& 324.21 & \textbf{129.09}
& 489.12 & \textbf{115.72} \\
& 30\%
& 694.71 & \textbf{116.61}
& 354.60 & \textbf{121.76}
& 559.89 & \textbf{107.61} \\
& 50\%
& 700.49 & \textbf{129.12}
& 369.51 & \textbf{131.71}
& 550.69 & \textbf{120.44} \\

\bottomrule
\end{tabular}
}
\end{table}

Across all three matched DP-SGD configurations,
$v$-pred., dynamic yields lower average MAE than
$\varepsilon$-prediction for all 20 dataset--severity conditions.
The method ordering is therefore preserved across the evaluated
configurations. Although absolute performance does not vary monotonically
across the evaluated configurations in every condition, the advantage of
the proposed method is maintained across
$\epsilon\in\{6,8,10\}$. This experiment provides a sensitivity analysis across matched
DP-SGD configurations with different privacy budgets.

\section{Controlled Gradient Diagnostic}
\label{app:controlled_gradient}

The training-time gradient diagnostics in
Section~\ref{sec:gradient_diagnostics} characterize the gradient distributions for individual private examples encountered during actual DP-SGD optimization.
Because those diagnostics are obtained from separately trained objective
variants, they do not isolate the immediate effect of timestep-aware
dynamic weighting at a fixed model state. We therefore conduct an
additional controlled gradient diagnostic in which the model state and
stochastic inputs are matched and only the weighting rule is changed from
$\lambda_t=1$ to $\lambda_t=\bar{\alpha}_t$.

\subsection{Diagnostic Setup}
\label{app:controlled_gradient_setup}

We fix a trained $v$-prediction checkpoint and keep the model state,
input batch, observation mask, diffusion noise, and timestep realizations
identical between two objective evaluations. The only difference is the
timestep weighting:
\[
v\text{-pred., unweighted},
\qquad
\lambda_t=1,
\]
and
\[
v\text{-pred., dynamic},
\qquad
\lambda_t=\bar{\alpha}_t.
\]

No parameter update is performed during the diagnostic. Let $G_u$ and
$G_d$ denote the pre-clipping aggregate gradients formed from the
corresponding unclipped gradients for the private examples under the
unweighted and dynamically weighted objectives, respectively. We record
their gradient norms and cosine similarity,
\begin{equation}
\operatorname{CosSim}(G_u,G_d)
=
\frac{G_u^\top G_d}
{\|G_u\|_2\|G_d\|_2}.
\label{eq:aggregate_cosine}
\end{equation}

When all examples receive the same scalar weight, gradient scaling does
not change the aggregate direction. In the present setting, different
examples are associated with different sampled diffusion timesteps and
therefore receive different values of $\lambda_t$. Timestep-aware
dynamic weighting can therefore change their relative contributions to
the aggregate gradient.

Because both evaluations use the same $v$-prediction target, this
diagnostic isolates the effect of timestep-aware dynamic weighting rather
than target reparameterization. The aggregate-gradient quantities
reported here are complementary to the training-time gradient diagnostics
in Section~\ref{sec:gradient_diagnostics}. In DP-SGD, clipping is applied
to the gradient for each private example before aggregation; accordingly,
the pre-clipping aggregate-gradient norms in this diagnostic are not used
to define clipping events or overall clipping fractions.

\subsection{Mixed-Timestep Diagnostic}
\label{app:mixed_t_gradient}

We first retain the mixed-timestep sampling used during training.
Table~\ref{tab:mixed_t_gradient} reports the aggregate-gradient
statistics obtained when only the weighting rule is changed.

\begin{table}[t]
\centering
\caption{
Mixed-timestep aggregate-gradient diagnostic.
The checkpoint, inputs, masks, diffusion noise, and sampled timesteps
are matched between unweighted $v$-prediction and $v$-prediction with
timestep-aware dynamic weighting; only the weighting rule is changed.
}
\label{tab:mixed_t_gradient}

\small
\setlength{\tabcolsep}{4.0pt}
\renewcommand{\arraystretch}{1.10}

\resizebox{0.98\linewidth}{!}{
\begin{tabular}{lccccc}
\toprule
\multirow{2}{*}{\textbf{Dataset}} &
\multicolumn{2}{c}{\textbf{Median Aggregate-Gradient Norm}} &
\multicolumn{2}{c}{\textbf{95th-Percentile Aggregate-Gradient Norm}} &
\multirow{2}{*}{\textbf{Gradient Cosine Similarity}} \\
\cmidrule(lr){2-3}
\cmidrule(lr){4-5}
&
\textbf{Unweighted} &
\textbf{Dynamic} &
\textbf{Unweighted} &
\textbf{Dynamic} &
\\
\midrule

ETTh1
& 4.748 & 2.123
& 14.099 & 14.066
& 0.774 \\

ETTm1
& 4.417 & 2.516
& 11.427 & 11.189
& 0.822 \\

Appliances
& 4.437 & 1.786
& 28.443 & 28.319
& 0.754 \\

\bottomrule
\end{tabular}
}
\end{table}

Timestep-aware dynamic weighting reduces the median aggregate-gradient
norm on all three datasets, while the change in the 95th-percentile
statistic is small under mixed-timestep sampling. The controlled comparison therefore reveals a timestep-dependent change
in the aggregate-gradient distribution induced by the weighting schedule.

The gradient cosine similarities are 0.774, 0.822, and 0.754 for
ETTh1, ETTm1, and Appliances, respectively. Since the model state and
stochastic inputs are matched between the two evaluations, these values
indicate that timestep-aware dynamic weighting changes the relative
contributions to the aggregate gradient rather than acting as a single
global scalar multiplier.

\subsection{Timestep-Stratified Diagnostic}
\label{app:timestep_gradient}

To localize the weighting effect along the diffusion schedule, we use
the implementation's zero-based timestep indexing and divide
$t=0,\ldots,999$ into ten consecutive intervals,
\[
0\text{--}99,\;
100\text{--}199,\;
\ldots,\;
900\text{--}999,
\]
and perform 30 controlled evaluations within each interval.
Table~\ref{tab:timestep_gradient} reports the complete results.

\begin{table}[t]
\centering
\caption{
Timestep-stratified aggregate-gradient diagnostic for unweighted
$v$-prediction and $v$-prediction with timestep-aware dynamic weighting.
Each interval contains 30 controlled evaluations with matched model
state and stochastic inputs between the two objective evaluations.
}
\label{tab:timestep_gradient}

\scriptsize
\setlength{\tabcolsep}{2.6pt}
\renewcommand{\arraystretch}{1.06}

\resizebox{0.98\linewidth}{!}{
\begin{tabular}{llccccc}
\toprule
\multirow{2}{*}{\textbf{Dataset}} &
\multirow{2}{*}{\textbf{Timestep Interval}} &
\multicolumn{2}{c}{\textbf{Median Aggregate-Gradient Norm}} &
\multicolumn{2}{c}{\textbf{95th-Percentile Aggregate-Gradient Norm}} &
\multirow{2}{*}{\textbf{Gradient Cosine Similarity}} \\
\cmidrule(lr){3-4}
\cmidrule(lr){5-6}
&
&
\textbf{Unweighted} &
\textbf{Dynamic} &
\textbf{Unweighted} &
\textbf{Dynamic} &
\\
\midrule

\multirow{10}{*}{ETTh1}
& 0--99
& 4.022 & 3.981
& 5.298 & 5.271
& 1.000 \\
& 100--199
& 4.047 & 3.833
& 9.552 & 9.467
& 1.000 \\
& 200--299
& 3.383 & 2.854
& 8.520 & 8.392
& 0.999 \\
& 300--399
& 3.489 & 2.440
& 9.362 & 9.094
& 0.997 \\
& 400--499
& 4.054 & 2.262
& 12.376 & 12.101
& 0.996 \\
& 500--599
& 5.015 & 2.085
& 11.172 & 10.985
& 0.992 \\
& 600--699
& 6.186 & 1.657
& 11.481 & 10.544
& 0.980 \\
& 700--799
& 7.559 & 1.161
& 10.149 & 8.208
& 0.966 \\
& 800--899
& 8.667 & 0.526
& 25.630 & 25.885
& 0.956 \\
& 900--999
& 8.510 & 0.094
& 21.530 & 18.065
& 0.882 \\

\midrule

\multirow{10}{*}{ETTm1}
& 0--99
& 5.190 & 5.132
& 8.074 & 8.027
& 1.000 \\
& 100--199
& 4.092 & 3.859
& 8.170 & 8.053
& 1.000 \\
& 200--299
& 3.583 & 3.094
& 9.312 & 9.217
& 0.999 \\
& 300--399
& 3.690 & 2.602
& 11.361 & 11.261
& 0.998 \\
& 400--499
& 5.005 & 2.882
& 12.740 & 12.577
& 0.997 \\
& 500--599
& 4.353 & 1.751
& 12.763 & 12.402
& 0.993 \\
& 600--699
& 5.693 & 1.465
& 15.541 & 15.132
& 0.989 \\
& 700--799
& 7.353 & 1.047
& 12.916 & 11.983
& 0.980 \\
& 800--899
& 8.345 & 0.518
& 13.286 & 11.616
& 0.955 \\
& 900--999
& 8.791 & 0.078
& 13.957 & 10.720
& 0.880 \\

\midrule

\multirow{10}{*}{Appliances}
& 0--99
& 3.658 & 3.628
& 5.569 & 5.537
& 1.000 \\
& 100--199
& 3.430 & 3.228
& 6.609 & 6.505
& 1.000 \\
& 200--299
& 2.900 & 2.526
& 6.354 & 6.225
& 0.999 \\
& 300--399
& 2.732 & 1.987
& 11.300 & 11.161
& 0.998 \\
& 400--499
& 3.438 & 1.929
& 17.743 & 17.375
& 0.997 \\
& 500--599
& 4.429 & 1.798
& 126.878 & 126.520
& 0.995 \\
& 600--699
& 7.245 & 2.027
& 9.961 & 9.204
& 0.984 \\
& 700--799
& 8.782 & 1.489
& 27.940 & 27.454
& 0.978 \\
& 800--899
& 7.934 & 0.662
& 14.874 & 11.798
& 0.941 \\
& 900--999
& 8.379 & 0.122
& 235.227 & 230.690
& 0.899 \\

\bottomrule
\end{tabular}
}
\end{table}

The timestep-stratified results show that the effect of timestep-aware
dynamic weighting depends on diffusion timestep. At $t=0$--$99$, the
gradient cosine similarity is 1.000 on all three datasets and the
aggregate-gradient norm statistics differ only slightly. The difference
becomes more pronounced toward the late low-SNR region.

The strongest attenuation of the median aggregate-gradient norm occurs
near the end of the diffusion schedule. For example, on ETTh1 at
$t=900$--$999$, the median norm decreases from 8.510 to 0.094.
ETTm1 and Appliances exhibit the same qualitative late-timestep pattern.

The gradient cosine similarity also decreases toward later diffusion
timesteps. In $t=900$--$999$, it is 0.882, 0.880, and 0.899 on
ETTh1, ETTm1, and Appliances, respectively. This schedule-dependent
behavior is consistent with timestep-aware dynamic weighting
$\lambda_t=\bar{\alpha}_t$, whose values become smaller toward the end
of the diffusion schedule.

The 95th-percentile statistic does not decrease in every interval. For
example, on ETTh1 at $t=800$--$899$, it changes from 25.630 to
25.885. The controlled measurements therefore indicate a
timestep-selective weighting effect rather than a uniform reduction of
every gradient statistic.

\subsection{Relation to Training-Time Gradient Diagnostics}
\label{app:controlled_gradient_summary}

The controlled gradient diagnostic and the training-time gradient
diagnostics capture complementary aspects of the optimization process.
Section~\ref{sec:gradient_diagnostics} characterizes the gradient
distribution for individual private examples encountered during actual
DP-SGD training, whereas the
controlled gradient diagnostic isolates the immediate effect of
timestep-aware dynamic weighting at a fixed model state. The
mixed-timestep results show that timestep-aware dynamic weighting changes
the relative contributions to the pre-clipping aggregate gradient, while
the timestep-stratified results show that this effect becomes more
pronounced toward late low-SNR diffusion timesteps. These observations
are consistent with the training-time gradient diagnostics and with the
illustrative timestep-wise gradient diagnostic in
Figure~\ref{fig:gradient_diagnostics}, and provide additional empirical
support for the theoretical mechanism characterized in
Section~\ref{sec:theory}.
%%%%%%%%%%%%%%%%%%%%%%%%%%%%%%%

\section*{Declaration of generative AI and AI-assisted technologies
in the manuscript preparation process}

During the preparation of this work, the authors used ChatGPT (OpenAI)
to assist with language refinement, readability improvement, and manuscript
organization. After using this tool/service, the authors reviewed and edited
the content as needed and take full responsibility for the content of the
published article.

\bibliographystyle{elsarticle-harv}
\bibliography{cas-refs}

\end{document}